\documentclass[letterpaper,twocolumn,10pt]{article}
\usepackage{usenix}
\usepackage{tikz}
\usepackage{cite}
\usepackage{algorithm}
\usepackage{algorithmic}
\usepackage{array}
\usepackage{url}
\usepackage{hyperref}
\usepackage{pifont}
\usepackage{amsfonts}
\usepackage{amsmath}
\usepackage{amsthm}
\usepackage{subfigure}
\usepackage{graphicx}
\usepackage{xcolor}
\usepackage{arydshln}
\usepackage{caption}
\usepackage{multirow}
\usepackage{booktabs}
\usepackage[normalem]{ulem}
\useunder{\uline}{\ul}{}
\definecolor{celestialblue}{rgb}{0.29, 0.59, 0.82}
\definecolor{cerulean}{rgb}{0.0, 0.48, 0.65}
\definecolor{cadmiumorange}{rgb}{0.93, 0.53, 0.18}

\newtheorem{thm}{Theorem}

\newtheorem{lemma}[thm]{Lemma}

\usepackage{tikz}

\newenvironment{paragraphs}[1]{
    \noindent\textbullet\,\textbf{#1 }\ignorespaces
}{}

\begin{document}

\date{}
\title{\Large \bf Provably Robust Adaptation for Language-Empowered Foundation Models}

\author{
    Yuni Lai\textsuperscript{\textdagger}, Xiaoyu Xue\textsuperscript{\textdagger}, Linghui Shen\textsuperscript{\textdagger}, 
    Yulun Wu\textsuperscript{\textsection}, Gaolei Li\textsuperscript{\textparagraph}, Song Guo\textsuperscript{*}, 
    Kai Zhou\textsuperscript{\textdagger}, Bin Xiao\textsuperscript{\textdagger}\\
    \textsuperscript{\textdagger}Department of Computing, The Hong Kong Polytechnic University\\
    \textsuperscript{\textsection}College of Systems and Engineering, National University of Defense Technology\\
    \textsuperscript{\textparagraph}School of Electronic Information and Electrical Engineering, Shanghai Jiao Tong University\\
    \textsuperscript{*}Department of Computer Science and Engineering, Hong Kong University of Science and Technology
}

\maketitle
\thispagestyle{empty}
\pagestyle{empty}
\begin{abstract}
Language-empowered foundation models (LeFMs), such as CLIP and GraphCLIP, have transformed multimodal learning by aligning visual (or graph) features with textual representations, enabling powerful downstream capabilities like few-shot learning. However, the reliance on small, task-specific support datasets collected in open environments exposes these models to poisoning attacks, where adversaries manipulate the support samples to degrade performance. Existing defenses rely on empirical strategies, which lack formal guarantees and remain vulnerable to unseen and adaptive attacks. Certified robustness offers provable guarantees but has been largely unexplored for few-shot classifiers based on LeFMs.

This study seeks to fill these critical gaps by proposing the \textit{first} provably robust few-shot classifier that is tailored for LeFMs. We term our model \textbf{L}anguage-\textbf{e}mpowered \textbf{F}ew-shot \textbf{C}ertification (\textbf{LeFCert}). It integrates both textual and feature embeddings with an adaptive blending mechanism. To achieve provable robustness, we propose a twofold trimmed mean prototype and derive provable upper and lower bounds for classification scores, enabling certification under worst-case poisoning scenarios. To further enhance the performance, we extend LeFCert with two variants by considering a more realistic and tighter attack budget: LeFCert-L incorporates randomized smoothing to provide Lipschitz continuity and derive robustness under dual budget constraints, and LeFCert-C provides collective certification for scenarios where attackers distribute a shared poisoning budget across multiple samples.
Experiments demonstrate that LeFCert achieves state-of-the-art performance, significantly improving both clean and certified accuracy compared to existing baselines. 
Despite its advanced robustness mechanisms, LeFCert is computationally efficient, making it practical for real-world applications. 

\end{abstract}

\section{Introduction}
\label{Sec:Intro}

\begin{figure}
    \centering
    \includegraphics[width=1\linewidth]{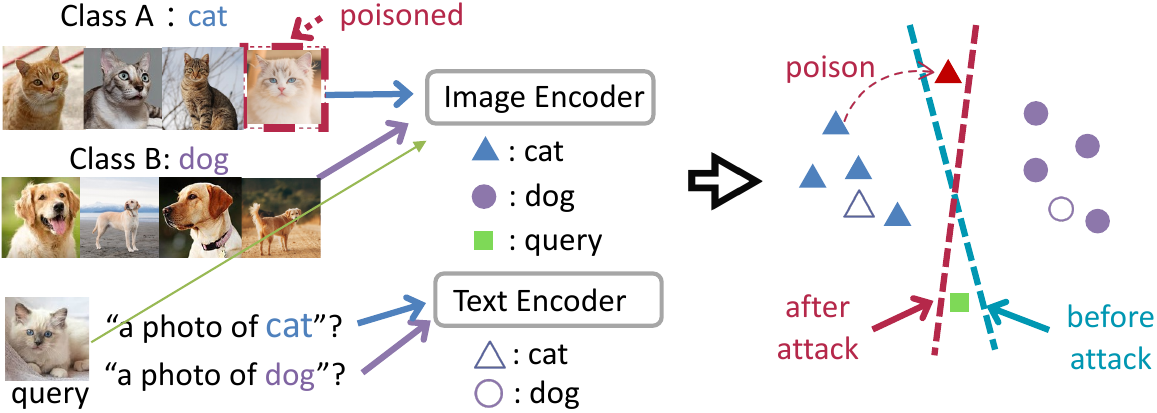}
    \caption{Poisoning attack on LeFM few-shot classifier.}
    \label{fig:poisoning_attack}
    \vspace{-10pt}
\end{figure}

The rapid advancement of \textbf{Language-empowered Foundation Models} (\textbf{LeFMs})~\cite{zhang2024vision,li2025survey}, such as CLIP~\cite{radford2021learning} (vision-language) and GraphCLIP~\cite{zhu2024graphclip} (graph-language model), has transformed multimodal learning by enabling seamless alignment between visual (or graph) features and textual representations. Different from the traditional supervised learning that relies on tremendous fine-grained labels~\cite{zhang2024vision}, these LeFM models achieve large-scale pre-training via image-text (or graph-text) pairs through contrastive learning. 
By aligning these modalities into a shared embedding space, these models unlock powerful downstream capabilities, especially \textbf{few-shot learning}~\cite{song2023comprehensive}. This has led to their widespread adoption in open platforms such as TorchMeta~\cite{deleu2019torchmeta}, LibFewShot~\cite{li2023libfewshot}, and LangChain~\cite{Chase2022LangChain}
, where users can readily adapt pre-trained models to diverse tasks by gathering small amounts of task-specific data for lightweight fine-tuning or prompt learning.

However, the adaptation phase often occurs in relatively open and uncontrolled environments. Task-specific data is frequently collected from diverse and potentially untrustworthy sources, introducing critical security concerns, particularly the risk of \textit{poisoning attacks}~\cite{tian2022comprehensive}. In few-shot learning scenarios, where models rely on a small labeled support set, an adversary can compromise the integrity of this curated data to degrade overall performance or induce targeted misclassifications~\cite{shafahi2018poison,oldewage2021attacking,oldewage2022adversarial,xu2021yet,alhussien2023novel,liu2024does}. The attacker can successfully poison the classifier by adding only a small perturbation to the support samples (as visualized in Figure~\ref{fig:poisoning_attack} and Figure~\ref{fig:pca}). The human imperceptible property of these attacks makes it difficult for the user to manually check the integrity of the support set~\cite{oldewage2022adversarial,wang2024fcert}.
Such poisoning attacks are especially problematic in safety-critical applications, where the reliability of predictions is paramount~\cite{fu2023styleadv,zhou2024fewshot,ghiasvand2025few}. 

\begin{figure*}[!t]
    \centering
    \includegraphics[width=0.9\linewidth]{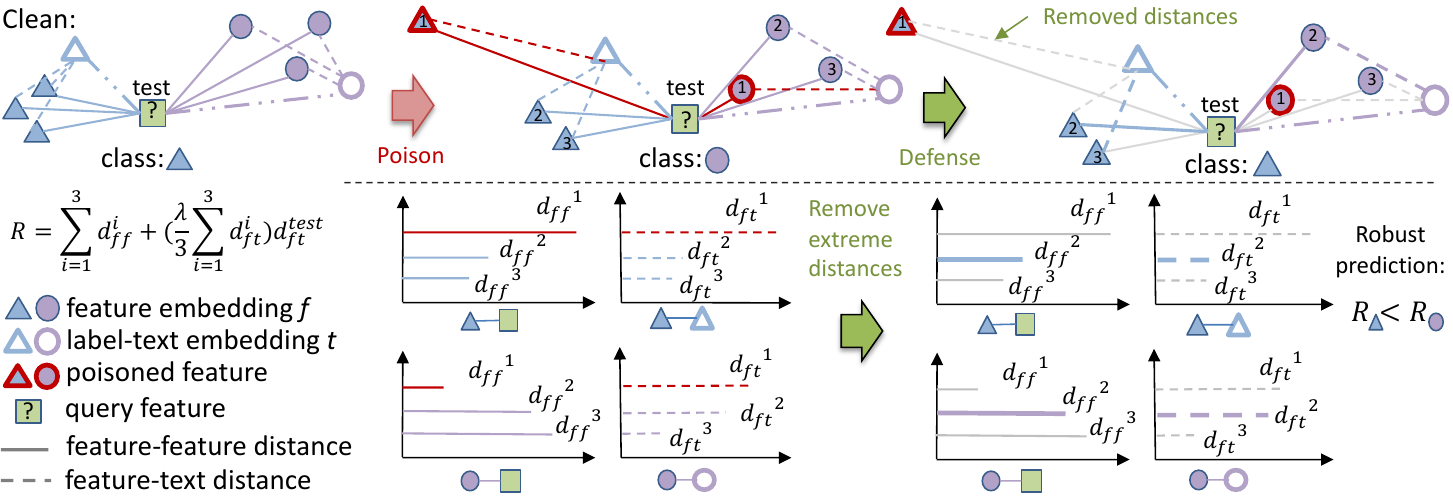}
    \caption{LeFCert ($C=2,K=3$). LeFCert introduces a novel hybrid prototype-based classifier that uniquely integrates textual and feature embeddings. It constructs robust prototypes by removing the highest and lowest distances (outliers). LeFCert provides provable robustness guarantees leveraging the upper and lower bounds of the integrated trimmed-mean prototype distances.}
    \label{fig:Overview}
    \vspace{-5pt}
\end{figure*}
To ensure robust and secure downstream applications of LeFMs, there are several limitations regarding the current few-shot adaptation approaches. Despite the development of few-shot adapters~\cite{alain2017understanding,gao2024clip,farina2025rethinking,zhang2022tip,huang2024lp++,silva2024closer} for LeFMs, their robustness to such adversarial manipulations during adaptation remains underexplored. 
Most existing efforts rely on \textit{empirical defense} strategies, such as adversarial training~\cite{goldblum2020adversarially,dong2022improving,zhou2024few} and robust parameter distillation~\cite{dong2024adversarially}. While these methods enhance robustness against \textit{known attacks}, they lack formal robustness guarantees. Consequently, they remain susceptible to \textit{adaptive} and \textit{unseen} attacks, which can exploit these defenses under stronger attacks. 

In contrast, \textit{certified robustness}, which provides \textit{provable guarantees} of a model's ability to withstand adversarial attacks, has recently gained significant traction in the machine learning community. Techniques such as randomized smoothing\cite{cohen2019certified,lecuyer2019certified,lai2024node,lai2024collective} and de-randomized (deterministic) approaches~\cite{levine2020randomized,levine2021deep,wang2022lethal,li2025agnncert} have emerged as powerful tools to certify model robustness under specific threat models. These \textbf{provable defense} makes sure the model prediction for a testing sample cannot be altered by \textit{arbitrary attacks} within a specified attack budget. 

While existing methods have been successfully applied to traditional vision and graph-based models that rely on supervised learning, certified robustness for few-shot classification based on LeFMs remains largely underexplored. 
Currently, FCert~\cite{wang2024fcert} is the only method tailored to certify the robustness of few-shot classifiers. However, it has not addressed the unique characteristics of LeFMs, such as their dual reliance on support samples and label text embeddings. By neglecting the textual embeddings that underpin the superiority of LeFMs, FCert is unable to fully leverage the strengths of these models (as evidenced by a significant drop in model accuracy when textual information is omitted, as shown in Figure~\ref{fig:images_Lambda}). 
This limitation leaves a critical gap in certified robustness for few-shot classifiers built on LeFMs, highlighting the need for novel approaches that can address these challenges while accounting for the hybrid nature of these models.


This study seeks to fill these critical gaps by proposing the \textit{first} provably robust few-shot classifier tailored for LeFMs. We termed our model as \textbf{L}anguage-\textbf{e}mpowered \textbf{F}ew-shot \textbf{Cert}ification (\textbf{LeFCert}). 
LeFCert integrates feature embeddings (from support samples) and textual embeddings (from label descriptions) into a hybrid prototype-based classifier (as illustrated in Figure~\ref{fig:Overview}). This integration is achieved through an adaptive blending mechanism that dynamically adjusts the contribution of textual and feature information based on their reliability. When support samples closely align with their corresponding label text embeddings, textual information is given more weight. 
This adaptive mechanism enables LeFCert to effectively utilize the complementary strengths of textual and feature embeddings to ensure high model accuracy.

To mitigate the influence of adversarial samples and provide certified robustness, LeFCert employs a twofold trimmed mean to the prototype, which discards the top-$M$ largest and smallest distances. By focusing on the most reliable subset of distances, this mechanism reduces the impact of outliers. Most importantly, we derive closed-form upper and lower bounds for classification distance scores, enabling the certification of predictions under worst-case adversarial attack. These bounds form the basis for a certified robustness condition, ensuring the model's predictions remain consistent under a predefined attack budget: the attacker can arbitrarily poison $T$ support samples. Our proposed LeFCert provides a \textit{deterministic} robustness guarantee with efficient computation. 

Furthermore, to address more realistic threat models and application scenarios, we extend LeFCert with two variants. LeFCert-L is designed for a dual-constraint threat model, where adversaries can perturb the $T$ support set, and at the same time, the perturbed samples are constrained to imperceptible perturbations within an $l_2$-norm ball. By incorporating randomized smoothing, LeFCert-L achieves Lipschitz continuity in the embedding function, ensuring bounded changes in the embedding space under small input perturbations. To further enhance performance, we propose LeFCert-LD, which employs diffusion denoise smoothing to mitigate the accuracy drop caused by Gaussian noise, achieving a better trade-off between clean and certified accuracy. LeFCert-C extends the certification framework to address collective adversarial scenarios, where an attacker distributes a shared poisoning budget across multiple testing samples. Unlike traditional sample-wise certification, which evaluates each test sample independently, LeFCert-C jointly certifies the robustness of multiple test samples by analyzing the worst-case budget allocation. This collective certification approach provides tighter robustness guarantees.

Comprehensive evaluations on benchmark datasets, including CIFAR-FS, Tiered-ImageNet, and CUB200-2011 (for image classification) and Cora and CiteSeer (for node classification), demonstrate that LeFCert significantly outperforms existing methods such as KNN~\cite{jia2022certified}, DPA~\cite{levine2021deep}, and FCert~\cite{wang2024fcert} in both clean accuracy and certified accuracy. 
For example, on CIFAR-FS (5-way, 10-shot), LeFCert achieves a clean accuracy of 98\%, outperforming FCert (88\%) and KNN (80\%). 
On Tiered-ImageNet, LeFCert achieves a certified accuracy of $92\%$ under a poisoning size of $T=3$, compared to FCert's 70\%. On graph datasets, LeFCert consistently achieves higher certified accuracy across diverse poisoning budgets, showcasing its robustness in both structured and unstructured data domains. LeFCert-LD demonstrates superior robustness under dual constraints. On Tiered-ImageNet with $T=7$, LeFCert-LD achieves a certified accuracy of 48\%, outperforming all baselines which is 0\%. Additionally, LeFCert-C achieves significant improvements in collective certification. For instance, on CIFAR-FS at $T=4$, LeFCert-C achieves 94\% certified accuracy, compared to 37\% for FCert. These results highlight the strength of LeFCert-C in modeling realistic adversarial constraints, providing stronger and tighter robustness guarantees. Despite its advanced robustness mechanisms,
LeFCert is computationally efficient. Across all the datasets, LeFCert can verify $5$ testing samples per episode within $0.9s$. LeFCert-C can verify $100$ testing samples within $7s$. 

By integrating textual and feature embeddings while addressing adversarial threats in few-shot learning, LeFCert sets a new benchmark for provably robust adaptation in language-empowered foundation models. Its superior performance, provable guarantees, and efficient computation make it a practical and robust choice for real-world applications where trust and reliability are crucial. 



\section{Background and Related Works}
\subsection{Language-Empowered Foundation Models}
Language-empowered foundation models, such as CLIP~\cite{radford2021learning} and GraphCLIP~\cite{zhu2024graphclip}, represent a significant advancement by aligning visual or graph features with textual semantic representations. These models leverage a dual-encoder architecture, where a \textbf{feature encoder} (e.g., a vision or graph encoder) extracts feature embeddings from vision or graph input, and a \textbf{language encoder} converts textual input into semantic embeddings. The feature encoder, denoted as $\mathcal{F}_{enc}$, maps input data from a high-dimensional input space to a lower-dimensional embedding space \( \mathbb{R}^d \), while the language encoder, \( \mathcal{T}_{enc} \), generates semantic embeddings in the same space \( \mathbb{R}^d \). Then the feature and text encoder is jointly pretrained by contrastive learning that matches the correct image-text pairs or graph-text pairs. 
This alignment allows the model to compute distance scores between input features and textual semantics, making them especially effective in scenarios with limited labeled data.


\subsection{Adaptation of Language-empowered FMs}
\subsubsection{Zero-shot Adaptation}
Language-empowered foundation models have shown significant promise in zero-shot learning. Without any labels, the model can classify new instances with several candidate label texts. Given a new instance $x_{\text{test}}$, zero-shot adapter calculates the similarity (e.g., cosine similarity) between the testing feature embedding $\mathbf{f}_{\text{test}}=\mathcal{F}_{enc}(x_{\text{test}})$ and the label text embedding corresponding to class $c$: $\mathbf{t}_c=\mathcal{T}_{enc}(\text{`A photo of a $\{y_c\}$'})$.  The predicted class is determined as:
\begin{equation}
\label{eqn:language_info}
    \hat{y} = \arg\max_c \, \mathbf{f}_{\text{test}}^\top \mathbf{t}_c,
\end{equation}
where $\mathbf{f}_{\text{test}}\in\mathbb{R}^d$ and $\mathbf{t}_c\in\mathbb{R}^d$.
Nevertheless, without any supervision on the downstream task, the performance is limited.

\subsubsection{Few-shot Adaptation}
A primary way of adapting pretrained foundation models to specific downstream tasks is via few-shot learning, where the goal is to classify new samples using only a small number of labeled examples (support set): $\mathcal{D}=\{(x_1,y_1),(x_2,y_2),\cdots,(x_n,y_n)\}$, where $x_i$ is the input sample, and $y_i\in\{1,2,\cdots, C\}$ is the label. For $C$-way and $K$-shot classification, the number of support samples is $n=CK$.


To achieve few-shot learning, there are several few-shot classifiers. The most representative one is feature-based adapter, such as ProtoNet~\cite{snell2017prototypical}, which calculates the similarity between the testing feature embedding $\mathbf{f}_{\text{test}}$ and the feature embeddings of the support samples belonging to each class. For class $c$, the \textit{class prototype} is computed by averaging the feature embedding in the support set: $\mathbf{p}_c:=\frac{1}{K}\sum_{i=1}^n y_{ic} \mathbf{f}_i,$ where $\mathbf{f}_i=\mathcal{F}_{enc}(x_i)$ is the feature embedding of the $i$-th support sample, $y_{ic} \in \{0,1\}$ indicates whether the $i$-th sample belongs to class $c$. The predicted class is determined as:
\begin{equation}
\label{eqn:feature_info}
    \hat{y} = \arg\max_c \mathbf{f}_{\text{test}}^\top \mathbf{p}_c.
\end{equation}
Linear probing~\cite{alain2017understanding} is another type of feature-only adapter. Instead of calculating the prototype, it trains a linear classifier based on the feature vectors of the support set. 
Although feature-only adapters are efficient in exploiting the features of the support set, they do not consider textual information from the language encoder.

The more advanced few-shot adapter leverages both the strengths of textual information in~\eqref{eqn:language_info} and feature embeddings in~\eqref{eqn:feature_info}. Prompt-based adapter, such as CoOp~\cite{zhou2022learning}, substitutes the fixed label-text with learnable textual tokens, and fine-tunes it on the pre-trained vision-language models with support samples. CLIP-Adapter~\cite{gao2024clip} trains two additional feature adapters and a residual module to fit the downstream task. CLAP~\cite{silva2024closer} initializes the weight of linear probing with a textual embedding prototype. 2SFS~\cite{farina2025rethinking} designs a two-stage few-shot classifier training based on visual and textual vectors. Tip-Adapter~\cite{zhang2022tip} and LP++~\cite{huang2024lp++} train a linear model and blend the image feature and textual feature. For instance, LP++ trains a linear classifier to integrate the visual and text embedding:
\begin{align}
S_c &= \mathbf{f}_{\text{test}}^\top \left(\mathbf{w}_c + \alpha_c \mathbf{t}_c \right),\\
\hat{y} &= \arg\max_c S_c,
\end{align}
where $\mathbf{w}_c$ are learnable model parameters that can be regarded as training-based class prototypes, and $\alpha_c$ is a learnable class-wise blending parameter that adjusts the contribution of the textual embedding. 

\textbf{Limitations}: While these classifiers focus on improving the test accuracy, the robustness against poisoning attacks remains unexplored. Furthermore, due to the complex model structure, it is not straightforward to provide a theoretical guarantee of their robustness. 

\subsection{Adversarial Attacks}
The data scarcity of few-shot learning increases its vulnerability to adversarial attacks~\cite{shafahi2018poison,oldewage2021attacking,oldewage2022adversarial,xu2021yet,alhussien2023novel,liu2024does}. Prior studies~\cite{goldblum2020adversarially,oldewage2021attacking} show that the most common projected gradient descent attack (PGD) can effectively achieve evasion and poisoning attacks to manipulate the few-shot prediction. Specifically, the attack ASP~\cite{oldewage2021attacking} optimizes the poisoned support set \( \mathcal{D}_p = \{ x^*, y \} \) by solving:
\begin{equation}
    \arg \max_{\delta} \mathcal{L}(f(x^*, g(x + \delta, y)), y^*) \quad \text{s.t.} \quad \|\delta\|_\infty < \epsilon,
\end{equation}
where \(x, y\) represent the support set inputs and labels, \(x^*, y^*\) the query inputs and labels, \( \delta \) the perturbation, and \( \epsilon \) the maximum perturbation range. Using PGD, ASP crafts adversarial support points that significantly reduce model accuracy, even when only a small fraction of the support set is poisoned. This highlights a critical vulnerability of few-shot learning systems, especially in sensitive applications.

\subsection{Existing Certified Defenses}
\label{sec:certified_baselines}
To defend against adversarial attacks, existing defense strategies can broadly be categorized into \textbf{empirical defenses} and \textbf{certified defenses}. Prior work mainly investigates empirical defenses based on adversarial training~\cite{goldblum2020adversarially,dong2022improving,zhou2024few}, robust parameter distillation~\cite{dong2024adversarially}. While these methods improve robustness against known attacks, they lack formal robustness guarantees and remain vulnerable to adaptive and unseen attacks, which can compromise their reliability under strong adversaries.

The development of certified defenses has gained traction as a critical approach to ensure robustness against adversarial manipulations~\cite{cohen2019certified,lecuyer2019certified,jia2021intrinsic,levine2021deep,jia2022certified,li2023sok,jia2023pore,wang2024fcert}. They provide provable guarantees of robustness against adversarial manipulations within a specified attack budget, ensuring robustness under the worst-case adversarial scenarios.

In few-shot adaptation, where the data is scarce or new classes arise frequently, certified approaches are required to adapt quickly. Below, we summarize three prominent approaches suitable for few-shot learning:

\begin{paragraphs}{DPA~\cite{levine2021deep}:} 
DPA divides the support set into multiple disjoint partitions using a hash function. Each partition builds an independent classifier, and the final prediction is made through majority voting. The robustness guarantee can be derived because $T$ poisoned samples can only affect at most $T$ classifiers.
\end{paragraphs}

\begin{paragraphs}{KNN~\cite{jia2022certified}:} The k-NN method certifies robustness by taking a majority vote over the \(k\) nearest neighbors of a query sample. Robustness is guaranteed because $T$ poisoned samples can only affect $T$ neighbors. 
\end{paragraphs}

\begin{paragraphs}{FCert~\cite{wang2024fcert}:} 
FCert certifies robustness by analyzing prototype distances of feature embeddings. It computes trimmed means by excluding extreme distances and evaluates upper and lower bounds of predictions under poisoning. As long as the mean distance upper bound is smaller than the lower bound, the prediction is unchanged. However, it is limited to feature embeddings and does not fully utilize textual embeddings in language-empowered models.
\end{paragraphs}

\textbf{Limitations}: There is a lack of certified defenses that integrate robustness guarantees with the unique properties of language-empowered foundation models, such as their reliance on textual embeddings and feature embeddings.





\section{Problem Statement}
\subsection{Threat model}

We assume that the attacker has all the knowledge about the data and model, and the attacker can arbitrarily poison at most $T$ support samples in $\mathcal{D}=\{(x_1,y_1),(x_2,y_2),\cdots,(x_n,y_n)\}$ (arbitrarily modify the $x_i$ or $y_i$ but ensure the number of $C$ classes and $K$ shots).
Let $T^c$ denote the number of poisoning sizes among the support sample with class $c$; we have $T=\sum_{c\in \{1,\cdots, C\}}T^c$. We denote the poisoned support set as $\mathcal{D}_p\in \mathcal{B}(\mathcal{D},T)$, where $\mathcal{B}(\mathcal{D},T)$ is the set of all possible poisoned support sets by manipulating $T$ samples among $\mathcal{D}$. 
By poisoning the support set, the attacker aims to manipulate the prediction result of a testing sample $x_{\text{test}}$. In realistic, there are a set of testing samples $\mathcal{X}_{\text{test}}=\{x_{\text{test}}^{(1)}, x_{\text{test}}^{(2)}, \ldots, x_{\text{test}}^{(N)}\}$ and the attacker aim to disrupt the predictions as many as possible.

\subsection{Goal of Defense}

In this paper, we focus on language-powered foundation models such as CLIP~\cite{radford2021learning} (Vision-Language model) and GraphCLIP~\cite{zhu2024graphclip} (Graph-Language model). The defender's goal is to provide certified robustness that ensures the classifier's prediction for a test sample remains unchanged under adversarial poisoning attacks on the support set, bounded by a predefined attack budget \( T \). Specifically, $\forall \mathcal{D}_p\in\mathcal{B}(\mathcal{D},T)$, our provable defense is designed to guarantee the robust prediction of few-shot classifier $g(x_{\text{test}};\mathcal{D})=g(x_{\text{test}};\mathcal{D}_p)$.

An ideal provable defense for few-shot classification should provide the following properties:
\begin{itemize}
    \item \textbf{High clean accuracy}: the model's accuracy on test samples when the support set is unperturbed, ensuring effective utilization of both feature embeddings and textual embeddings.
    \item \textbf{High certified accuracy}: the ratio of test samples that are both correctly classified and provably robust against adversarial poisoning attacks within the attack budget \( T \).
\end{itemize}
By achieving these objectives, our defense aims to balance clean performance and robustness, ensuring the reliability and safety of language-empowered foundation models in adversarial settings.

\section{Provably Robust Adaptation: LeFCert}
In this section, we first propose a language-integrated and certifiably robust classifier. Then we analyze the classification bounds and provide the provably robust condition. Furthermore, we introduce two variants that consider a more realistic threat model to enhance the certified performance.

\subsection{Robust Language-Empowered Adapter}

Few-shot classifiers are increasingly used in scenarios where data is scarce. However, the lack of robust adaptation mechanisms leaves them vulnerable to adversarial perturbations, especially in contexts where both feature embeddings and label text embeddings can be leveraged for classification. While recent works have explored advanced few-shot adapters, the development of certifiably robust few-shot adapters remains underexplored.
To address this gap, we propose a robust language-empowered classifier that integrates both feature and textual information, which enables the derivation of provable upper and lower bounds for classification scores.

\subsubsection{Base classifier} Inspired by Tip-Adapter~\cite{zhang2022tip} and LP++~\cite{huang2024lp++}, the key idea is to find a suitable way to properly blend the textual information $\mathbf{f}_{\text{test}}^\top \mathbf{t}_c$ in~\eqref{eqn:language_info} and feature information $\mathbf{f}_{\text{test}}^\top \mathbf{p}_c$ in~\eqref{eqn:feature_info}. In our paper, we employ a training-free but adaptive and flexible class-wise blending parameter $\alpha_c$ based on the rationale that: 
if the support samples in a class are closer to their label text embedding, the textual embedding provides more reliable semantic information and should have a stronger influence. Conversely, if the support samples deviate significantly from their label embedding, their contribution should be reduced. Specifically, we employ the blending coefficient as follows:
\begin{equation}
    \alpha_c:=\frac{\lambda}{K} \sum_{i=1}^n y_{ic} \mathbf{f}_i^\top \mathbf{t}_c, \quad c\in\{1,2,\cdots,C\},
\end{equation}
where $\lambda$ is a hyperparameter that globally controls the contribution of the text knowledge; $\alpha_c$ is an adaptive and class-wise blending parameter;
$\mathbf{f}_i$ is the $i^{th}$ support sample, and $y_{ic}=1$ if $y_i=c$. We note that there are $K$ support samples that belong to class-$c$. To further simplify the nation, let $\mathbf{f}_{ci}$ denotes the $i^{th}$ sample in class $c$, we define our initial training-free integrated classification scores as:
\begin{align}
    S_c &= \mathbf{f}_{\text{test}}^\top \mathbf{p}_c+\alpha_c \mathbf{f}_{\text{test}}^\top \mathbf{t}_c,\\
    &=\sum_{i=1}^K \mathbf{f}_{\text{test}}^\top \mathbf{f}_{ci}+\frac{\lambda}{K} (\sum_{i=1}^K \mathbf{f}_{ci}^\top \mathbf{t}_c) \mathbf{f}_{\text{test}}^\top \mathbf{t}_c.\label{eqn:lp++}
\end{align}
Next, the design of our robust classifier contains two parts: ensuring the robustness of the prototypes $\mathbf{p}_c$, and the robustness of the blending coefficient $\alpha_c$. 

Essentially, these vector multiplications $\mathbf{f}_{\text{test}}^\top \mathbf{f}_{ci}$ and $\mathbf{f}_{ci}^\top \mathbf{t}_c$ in Eq.~\eqref{eqn:lp++} compute cosine similarities. To generalize the few-shot classifier, we use distance measurement to substitute the similarity measurement. For the convenience of analyzing the lower bound and upper bound of the classification scores, 
we ensure that all the distances are nonnegative. In general, we represent the distance between vectors $A$ and $B$ as $d(A,B)$, where $d(A,B)\geq 0, \forall A,B\in \mathbb{R}^d$. We employ the cosine distance $d(A,B)=1-\frac{A^\top B}{||A||_2 ||B||_2}\in[0,2]$. Similarly, we can also use other distances such as $l_2$-norm distance $d(A,B)=||A-B||_2\in[0,\infty]$. Then, we have the more general classification scores:
\begin{equation}
    S_c=-\sum_{i=1}^K d(\mathbf{f}_{\text{test}}, \mathbf{f}_{ci})-\frac{\lambda}{K} (\sum_{i=1}^K d(\mathbf{f}_{ci}, \mathbf{t}_c)) d(\mathbf{f}_{\text{test}}, \mathbf{t}_c).
\end{equation}
This classification score can be regarded as a hybrid prototype-based classifier that adaptively integrates the texture center and feature center.  

\subsubsection{Robust Classifier}
To ensure the robustness of the classification scores, we aim to mitigate the influence of outliers (e.g., adversarial samples). In the sum of distances $\sum_{i=1}^K d(\mathbf{f}_{\text{test}}, \mathbf{f}_{ci})$ and $\sum_{i=1}^K d(\mathbf{f}_{ci}, \mathbf{t}_c)$, we regard the largest and smallest top-$M$ distances as the outliers, and we average the remaining distances. Specifically, for a class $c$,
let $p_1^c\geq p_2^c\geq\cdots\geq p_K^c$ denote the sorted sequence of $d(\mathbf{f}_{\text{test}},\mathbf{f}_{ci})$, $i=1,\cdots,K$, and $q_1^c\geq q_2^c \geq \cdots \geq q_K^c$ denote the sorted sequence of $d(\mathbf{f}_{ci},\mathbf{t}_c)$, $i=1,\cdots,K$. We remove the $M$ largest and $M$ smallest of $p_i^c$ and $q_i^c$, respectively. $M\leq \lfloor (K-1)/2\rfloor$ is a hyperparameter. Then, we formally define the robust classifier $g(x_{\text{test}};\mathcal{D})$ as: 
\begin{align}
\label{eqn:robust_classifier}
\hat{y} &=g(x_{\text{test}};\mathcal{D})= \arg\min_c R^c,\\
R^c &= \sum_{i=M+1}^{K-M} p_i^c+\frac{\lambda}{K-2M} (\sum_{i=M+1}^{K-M} q_i^c) d(\mathbf{f}_{\text{test}},\mathbf{t}_c),\label{eqn:R^c}
\end{align}
where $R^c$ is the robust classification score for class $c$. These scores can be regarded as calculating the distance of the query sample to the twofold trimmed mean prototypes. We further visualize the LeFCert in Figure~\ref{fig:Overview}, and provide an additional example in Figure~\ref{fig:pca} (Appendix~\ref{sec:more_results}) that visualizes real embeddings with Principal Component Analysis (PCA) to show how LeFCert can mitigate the power of poisoning attacks. 

\subsection{Upper and lower bound of $R^c$}
Next, we show that the robust classification score $R^c$ based on twofold trimmed mean has a closed-form upper bound and lower bound, which paves the way for deterministic and efficient certification.
\label{sec:normal_bounds}
\begin{thm}
\label{thm:bounds}
Let $p_1^c\geq p_2^c\geq\cdots\geq p_K^c$ denote the sorted sequence of $d(\mathbf{f}_{\text{test}},\mathbf{f}_{ci})$, $i=1,\cdots,K$, and $q_1^c\geq q_2^c \geq \cdots \geq q_K^c$ denote the sorted sequence of $d(\mathbf{f}_{ci},\mathbf{t}_c)$, $i=1,\cdots,K$. Let $R^c$ be the classification score defined in Eq.~\eqref{eqn:R^c}. Suppose the attacker can arbitrarily modify $T^c$ features of the support samples among $\{\mathbf{f}_{ci}|i=1,\cdots, K\}$. If $T^c\leq M$, we have the upper bound and lower bound of the classification score $R^c$ for a given class $c$ and a given perturbation size $T^c$:
$$\overline{R}^c(T^c)=\sum_{i=M+1+T^c}^{K-M+T^c} p_i^c +\frac{\lambda}{K-2M} (\sum_{i=M+1+T^c}^{K-M+T^c} q_i^c) d(\mathbf{f}_{\text{test}},\mathbf{t}_c).$$
$$\underline{R}^c(T^c)=\sum_{i=M+1-T^c}^{K-M-T^c} p_i^c +\frac{\lambda}{K-2M} (\sum_{i=M+1-T^c}^{K-M-T^c} q_i^c) d(\mathbf{f}_{\text{test}},\mathbf{t}_c).$$
\end{thm}

If $M<T^c\leq K-M-1$, we can employ the interval of cosine distance $d(A,B)=1-\frac{A^\top B}{||A||_2 ||B||_2}\in[0,2]$ to obtain the bounds:
\begin{align}
    \overline{R}^c(T^c)&=(\sum_{i=M+1+T^c}^{K} p_i^c+2(T^c-M) )+\\
    &\frac{\lambda}{K-2M} (\sum_{i=M+1+T^c}^{K} q_i^c+2(T^c-M)) d(\mathbf{f}_{\text{test}},\mathbf{t}_c).\nonumber
\end{align}
$$\underline{R}^c(T^c)=\sum_{i=1}^{K-M-T^c} p_i^c +\frac{\lambda}{K-2M} (\sum_{i=1}^{K-M-T^c} q_i^c) d(\mathbf{f}_{\text{test}},\mathbf{t}_c).$$
If the distance $d(\cdot,\cdot)$ does not have a box constraint, there are no finite bounds, and we can not certify any sample for $M<T^c\leq K-M-1$. When $T^c>K-M-1$, we cannot certify any sample. We provide the proof in Appendix~\ref{sec:proofs}. 

\subsection{Provable robustness via optimization}
With the upper bound and lower bound of the classification score $R^c(T^c)$ for a given class $c$ and its perturbation size $T^c$, finally, we can obtain the provably robust condition. To find the worst-case attacker, we employ a simple optimization problem that optimizes the allocation of perturbation size $T$ into various classes for each query instance. 
\begin{thm}
(Provably Robust Condition). Let $\mathcal{D}$ denote the clean support set for $C$-way $K$-shot few classification. The few-shot classifier $g$ is defined in \eqref{eqn:robust_classifier}. Let $\hat{y}=g(x_{\text{test}};\mathcal{D})$ represent the prediction for testing input $\text{test}$ with $\mathcal{D}$ as the support set.
We have a provably robust classification result:
\begin{align}
    g(x_{\text{test}};\mathcal{D})=g(x_{\text{test}};\mathcal{D}_p),\, \forall \mathcal{D}_p\in \mathcal{B}(\mathcal{D},T),
\end{align}
if:
\begin{align}
\label{eqn:certify_condition}
\overline{R}^{\hat{y}}(T^{\hat{y}})<\min_{c\neq \hat{y}} \underline{R}^c(T-T^{\hat{y}}), \forall T^{\hat{y}}: 0\leq T^{\hat{y}}\leq T.
\end{align}

\end{thm}
We provide the proof in Appendix~\ref{sec:proofs}. We termed our proposed language-empowered few-shot certification model as \textbf{LeFCert}, and we further summarize it in Algorithm~\ref{alg:lefcert}.

\begin{algorithm}[!ht]
\caption{LeFCert}
\label{alg:lefcert}
\begin{algorithmic}[1]
\REQUIRE $C$-way $K$-shot support set $\mathcal{D}$, test input $x_{\text{test}}$, trimming parameter $M$, attack budget $T$, hyperparameter $\lambda$.
\ENSURE Predicted class $\hat{y}$ and certified robustness status.

\STATE \textbf{Step 1: Encode the features and text.}
\FOR{each classes $c \in \{1, 2, \ldots, C\}$}
    \STATE $\mathbf{f}_{ci} \gets \mathcal{F}_{enc}(x_{ci}),\quad i = 1, \ldots, K$\\
    $\mathbf{t}_c \gets \mathcal{T}_{enc}(\text{`A photo of a $\{y_c\}$'})$
\ENDFOR
\STATE $\mathbf{f}_{\text{test}} \gets \mathcal{F}_{enc}(x_{\text{test}})$
\STATE \textbf{Step 2: Compute and sort distance scores.}
\FOR{each class $c$}
    \STATE $p_i^c \gets d(\mathbf{f}_{\text{test}}, \mathbf{f}_{ci}), \quad i = 1, \ldots, K$
    \STATE $q_i^c \gets d(\mathbf{f}_{ci}, \mathbf{t}_c), \quad i = 1, \ldots, K$
    \STATE 
    $\mathbf{p}^c \gets \text{sort} (\mathbf{p}^c)$,
    $\mathbf{q}^c \gets \text{sort} (\mathbf{q}^c)$
\ENDFOR
\STATE \textbf{Step 3: Integrated robust classification score.}
\FOR{each class $c$}
    \STATE Compute the twofold trimmed robust classification score $R^c$ according to Eq.~\eqref{eqn:R^c}.
\ENDFOR
\STATE \textbf{Step 4: Upper and Lower bounds analysis.}
\FOR{each class $c$ and perturbation size $1\leq T^c\leq T$}
    \STATE Compute upper bound $\overline{R}^c(T^c)$ and lower bound $\underline{R}^c(T^c)$ according to Section~\ref{sec:normal_bounds}.
\ENDFOR
\STATE \textbf{Step 5: Robust prediction and certified robustness.}
\STATE Predict the class: $\hat{y} \gets \arg\min_c R^c$.
\STATE Verify certified robustness condition in \eqref{eqn:certify_condition}:
\STATE $Certified\gets \text{True}$
\FOR{For $T^{\hat{y}}\in\{0,\cdots,T\}$}
\FOR{For each class $c\neq \hat{y}$}
\IF{$\overline{R}^{\hat{y}}(T^{\hat{y}})>\underline{R}^c(T - T^{\hat{y}})$}
\STATE $Certified\gets \text{False}$
\ENDIF
\ENDFOR
\ENDFOR
\RETURN Predicted class $\hat{y}$ and its status $Certified$.
\end{algorithmic}
\end{algorithm}

\section{Extensions: Enhancing Certified Robustness}

In this section, we introduce two extensions of our model designed to enhance certified robustness under more realistic threat models and practical settings. These extensions, \textbf{LeFCert-L} and \textbf{LeFCert-C}, address different challenges: the first focuses on constraining adversarial perturbations within an $l_2$-norm ball, while the second extends certification to collective adversarial scenarios.

\subsection{Variant-1: Dual-constraint Certification}

Existing certifications~\cite{levine2021deep,jia2022certified,wang2024fcert}, including our basic version of LeFCert, are limited to handling a single-level budget constraint. This allows the attacker to arbitrarily modify the chosen $T$ poisoned samples without considering imperceptibility in the certification process.

In this variant, we consider a more realistic threat model with \textit{dual constraint}, where the adversary is allowed to poison $T$ samples but is also restricted to perturbations within an $l_2$-norm ball, ensuring the perturbations remain imperceptible. 
Formally, the adversary is constrained as follows:
\begin{equation}
    \forall x'_i \in \mathcal{D}_p, ||x_i-x'_i||_2\leq r,
\end{equation}
where $x'_i$ is the perturbed sample, and $r$ is predefined radius. 

Directly associating the input-level $l_2$ constraints with the constraints of feature embedding is challenging due to the black-box nature and complexity of deep learning encoders. To tackle the challenge, we employ randomized smoothing to create a smoothed encoder with Lipschitz continuity. 

Randomized smoothing~\cite{cohen2019certified} transforms a base classifier $f(x)$ to a smoothed classifier $f_s(x)$ by adding isotropic Gaussian noise to the input: $\epsilon \sim \mathcal{N}(x, \sigma^2 I)$. That is, given a classifier $f: \mathbb{R}^D \rightarrow [0, 1]$ and smoothing distribution $\mathcal{N}(0, \sigma^2 I)$, the classifier $g$ is defined as follows:

\begin{equation}
    f_s(x) = \frac{1}{(2\pi\sigma^2)^{n/2}} \int_{\mathbb{R}^D} f(x + \epsilon) \exp\left(-\frac{\|\epsilon\|_2^2}{2\sigma^2}\right) d\epsilon.
    \tag{5}
\end{equation}
Building on this, Pautov et al.~\cite{pautov2022smoothed} extended randomized smoothing to certify the embedding function and derive its Lipschitz continuity:
\begin{thm}
(Pautov et al.\cite{pautov2022smoothed}) Suppose that $f: \mathbb{R}^D \to \mathbb{R}^d$ is a deterministic embedding function and its smoothed function with Gaussian noise 
\begin{equation}
    f_s(x) = \mathbb{E}_{\epsilon \sim \mathcal{N}(0, \sigma^2 I)} f(x + \epsilon)
\end{equation}
is continuously differentiable for all $x$. If for all $x$, $\|f(x)\|_2 = 1$, then $f_s(x)$ is $L$--Lipschitz in $l_2$--norm:
\begin{equation}
    \forall x, x' \in \mathbb{R}^D, \|f_s(x) - f_s(x')\|_2 \leq L \|x - x'\|_2,
\end{equation}
where $L = \sqrt{\frac{2}{\pi \sigma^2}}$.
\end{thm}

Using the feature encoder $\mathcal{F}_{enc}$ in the foundation model as the base embedding function, we can create our smoothed encoder model $\mathcal{F}_s(x)$:
\begin{equation}
    \mathcal{F}_s(x) = \mathbb{E}_{\epsilon \sim \mathcal{N}(0, \sigma^2 I)} \mathcal{F}_{enc}(x + \epsilon),
\end{equation}
where the $\mathcal{F}_{enc}$ is a vision feature encoder or graph feature encoder. To ensure that for all $x$, $\|\mathcal{F}_{enc}(x)\|_2 = 1$, we apply normalization to the output embedding. Then, $\forall x, x' \in \mathbb{R}^D$ we have the constrain of the smoothed feature embedding:
\begin{equation}
    \|\mathcal{F}_s(x) - \mathcal{F}_s(x')\|_2 \leq L \|x - x'\|_2\leq Lr.
\end{equation}


With the Lipschitz continuity of the smoothed encoder, we can refine the upper and lower bounds of the robust classification scores $R^c$ when using $l_2$-norm as the distance measurement. 
To obtain the upper bound of $R^c$ under the Lipschitz continuity $\overline{R}_L^c(T^c)$, the problem becomes that the attacker needs to select the optimal $T^c$ index of $K$ support samples, and enlarge $\{p_i^c\}$ and $\{q_i\}$ by $Lr$, for $i\in\mathcal{I}$, where $\mathcal{I}$ denotes the selected index, $p_i$ represents the distance $d(\mathbf{f}_{\text{test}},\mathbf{f}_{ci})$, and $q_i$ represents $d(\mathbf{f}_{ci},\mathbf{t}_{c})$.
We can obtain the bounds for $R^c$ by the traversal method. The traversal approach is further detailed in Algorithm~\ref{alg:taversal}. The complexity $O(C(K, T^c))$ is reasonable for small $T^c$ and 
$K$. We termed this variant of our model as \textbf{LeFCert-L}.


\begin{algorithm}[!ht]
\caption{Traversal Bound (for LeFCert-L)}
\label{alg:taversal}
\begin{algorithmic}[1]
\REQUIRE Unsorted distance lists $\mathbf{p},\mathbf{q} \in \mathbb{R}^K$ for $c$-class support set, trimming parameter $M$, attack budget $T^c$, modification range $Lr$.
\ENSURE Optimal upper bound $\overline{R}_L^c(T^c) \in \mathbb{R}$.
\STATE Initialize $\overline{R}_L^c(T^c) \gets -\infty$
\STATE Generate all combinations of $T^c$ indices from $\{1, 2, \dots, K\}$: $\mathcal{I} \gets \text{combinations}(K, T^c)$
\FOR{each index set $\mathbf{i} \in \mathcal{I}$}
    \STATE Clone $\mathbf{p}$ to $\mathbf{p}^{\text{mod}}$, and  $\mathbf{q}$ to $\mathbf{q}^{\text{mod}}$.
    \FOR{$k \in \mathbf{i}$}
        \STATE $\mathbf{p}^{\text{mod}}[k] \gets \mathbf{p}^{\text{mod}}[k] + Lr$
        \STATE $\mathbf{q}^{\text{mod}}[k] \gets \mathbf{q}^{\text{mod}}[k] + Lr$
    \ENDFOR
    \STATE Resort: $\mathbf{p}^{\text{mod}} \gets \text{sort}(\mathbf{p}^{\text{mod}})$, $\mathbf{q}^{\text{mod}} \gets \text{sort}(\mathbf{q}^{\text{mod}})$
    \STATE $\mathbf{p}_t \gets \mathbf{p}^{\text{mod}}[M:K-M]$
    \STATE $\mathbf{q}_t \gets \mathbf{q}^{\text{mod}}[M:K-M]$
    \STATE $R^{mod}\gets\text{sum}(\mathbf{p}_t)+\lambda\text{mean}(\mathbf{q}_t)d(\mathbf{f}_{\text{test}},\mathbf{t}_c)$
    \STATE $\overline{R}_L^c(T^c) \gets \max(R^{mod},\overline{R}_L^c(T^c))$
\ENDFOR
\RETURN $\overline{R}_L^c(T^c)$
\end{algorithmic}
\end{algorithm}

Nevertheless, the random Gaussian noise led to a drop in accuracy (Table~\ref{tab:cert_images}). We further employ Diffusion Denoise Smoothing~\cite{carlini2023certified}, applying off-the-shelf diffusion models to denoise the input, and we termed this variant of our model as \textbf{LeFCert-LD}. Specifically, the denoised and smoothed encoder is defined as: 
\begin{equation}
    \mathcal{F}_{ds}(x) = \mathbb{E}_{\epsilon \sim \mathcal{N}(0, \sigma^2 I)} \mathcal{F}_{enc}(denoise(x_t;t)),
\end{equation}
where $t$ is the time-step in the diffusion model such that $\sigma^2=\frac{1-\alpha_t}{\alpha_t}$, the $\alpha_t$ is scheduled by the diffusion model that determines the intensity of Gaussian noise to be added, and $x_t=\sqrt{\alpha_t} (x+\epsilon)$.

\subsection{Variant-2: Collective Certification}
\begin{figure}[!h]
    \centering
    \includegraphics[width=1\linewidth]{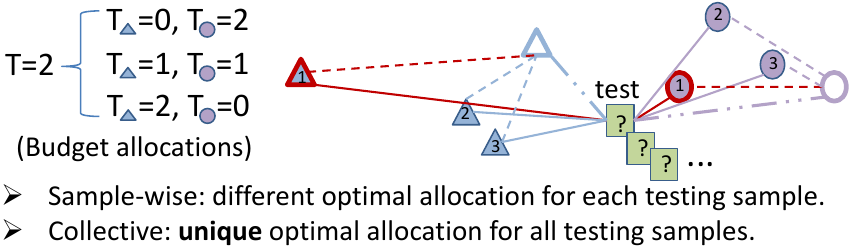}
    \caption{Illustration of collective certification (LeFCert-C).}
    \label{fig:collective}
\end{figure}

Existing certification frameworks~\cite{wang2024fcert,levine2021deep} and the current LeFCert assume a \textit{sample-wise certification} approach, where each testing sample is certified independently, which inherently makes the assumption that the attacker can consume all budget $T$ to disrupt the prediction of each test sample. However, in reality, the attacker can only allocate $T$ poisoned samples to disrupt a set of samples. It means that the adversary needs to allocate the poisoning budget $T$ across the support set to maximize the number of misclassified test samples. We note that sample-wise and collective certification are under the same threat model. However, because sample-wise certification cannot model this realistic attack budget, it has to relax the constraints (illustrated in Figure~\ref{fig:collective}). 

To address this limitation, we propose \textbf{LeFCert-C} (collective certification), which certifies the robustness of multiple testing samples simultaneously by analyzing the worst-case allocation of the poisoning budget. Comparing the sample-wise approach, the collective certificate can provide tighter bounds for the certified accuracy regarding a set of testing instances.

\begin{algorithm}[ht]
\caption{LeFCert-C: Collective Certification}
\label{alg:lefcert_c}
\begin{algorithmic}[1]
\REQUIRE Testing set $\mathcal{X}_{\text{test}} = \{x_{\text{test}}^{(1)}, \ldots, x_{\text{test}}^{(N)}\}$, support set $\mathcal{D}$, poisoning budget $T$, number of classes $C$, robust classification bounds $\underline{R}^c(T^c)$ and $\overline{R}^c(T^c)$, for $c\in\{1,\cdots,C\}$, and $T^c\in\{1,\cdots,T\}$.
\ENSURE Maximum number of misclassified samples $B$, optimal budget allocation $\{T^c\}$.
\STATE \textbf{Step 1: Enumerate All Budget Allocations.}
\STATE Generate all possible allocations $\{T^c\}$ of the poisoning budget $T$ across $C$ classes, satisfying:
$\sum_{c=1}^C T^c \leq T.$
\STATE \textbf{Step 2: Evaluate Non-certified Number.}
\STATE Initialize $B_{\text{max}} \gets 0$ and $\{T^c_{\text{opt}}\} \gets \emptyset$.
\FOR{each budget allocation $\{T^c\}$}
    \STATE Initialize $b_i \gets 0$ for all $i \in \{1, \ldots, N\}$.
    \FOR{each test sample $x_{\text{test}}^{(i)}$}
        \IF{$\overline{R}^{\hat{y}}(T^{\hat{y}})>\min_{c\neq \hat{y}} \underline{R}^c(T-T^{\hat{y}})$}
            \STATE Set $b_i \gets 1$.
        \ENDIF
    \ENDFOR
    \STATE Compute the sum: $B \gets \sum_{i=1}^N b_i$.
    \IF{$B > B_{\text{max}}$}
        \STATE Update $B_{\text{max}} \gets B$ and $\{T^c_{\text{opt}}\} \gets \{T^c\}$.
    \ENDIF
\ENDFOR

\STATE \textbf{Step 3: Return Optimal Results.}
\RETURN $B_{\text{max}}$, $\{T^c_{\text{opt}}\}$.
\end{algorithmic}
\end{algorithm}

The key distinction in LeFCert-C lies in the optimization of the attacker's budget allocation. Instead of certifying each test sample in isolation, LeFCert-C evaluates the worst-case distribution of the poisoning budget $T$ across classes and support samples $\mathcal{D}$, aiming to maximize the number of misclassified (not certified) test samples $\mathcal{X}_{\text{test}} = \{x_{\text{test}}^{(1)}, x_{\text{test}}^{(2)}, \ldots, x_{\text{test}}^{(N)}\}$. This leads to a combinatorial optimization problem, where the attacker seeks the optimal allocation of $T$ poisoned samples across classes and their corresponding modifications to $\mathbf{f}_{ci}'$, the feature embeddings of poisoned support samples.

Formally, the attacker’s objective is to maximize the number of misclassified test samples:
\begin{align}
    \max_{\{T^1,T^2,\cdots.T^C\}} B:=\sum_{i=1}^N b_i,
\end{align}
where $b_i = 1$ if the robust classification condition~\eqref{eqn:certify_condition} fails for the $i$-th test sample:
$\overline{R}^{\hat{y}}(T^{\hat{y}})>\min_{c\neq \hat{y}} \underline{R}^c(T-T^{\hat{y}})$,
and $b_i = 0$ otherwise. Here, $T^c$ represents the number of poisoned samples allocated to class $c$, satisfying $\sum_{c=1}^C T^c \leq T$, and $\overline{R}$ and $\underline{R}$ denote the lower and upper bounds of the robust classification scores. This optimization problem is to find the worst-case budget allocation, and we propose a traversal algorithm to solve it, which is detailed in Algorithm~\ref{alg:lefcert_c}. 

LeFCert-C certifies the collective robustness by ensuring that, even under the worst-case allocation of $T$, the predictions of at least $N-B_{max}$ test samples remain robust, where $B_{max}$ is the maximum number of samples that can be misclassified in the worst-case. 

LeFCert-C extends the applicability of LeFCert in modeling a collective budget constraint, providing a stronger and tighter certification framework. By jointly considering multiple test samples and optimizing over budget allocations, LeFCert-C ensures a higher certifiably robust ratio (Table~\ref{tab:cert_collect}).

\section{Experiments}

In this section, we comprehensively evaluate the effectiveness of LeFCert and its variants by addressing the following research questions:

\begin{itemize}
    \item \textbf{Q1 [Certified Accuracy]:} How effective is LeFCert and its variants in achieving certified robustness compared to state-of-the-art baselines?
    \item \textbf{Q2 [Clean Accuracy]:} How well does LeFCert balance the trade-off between clean and certified accuracy?
    \item \textbf{Q3 [Efficiency]:} What is the computational efficiency of LeFCert and its variants, and are they practical for few-shot prediction?
    \item \textbf{Q4 [Parameter Analysis]:} How does each parameter or component influence the performance of LeFCert?
\end{itemize}

\subsection{Experiment settings}

\subsubsection{Datasets}

In this work, we evaluate our few-shot certification framework on standard benchmark datasets covering both vision and graph domains.

\begin{paragraphs}{Image Datasets:}
Our evaluation employs three standard few-shot learning benchmarks with class-disjoint partitions: 
(1) CIFAR-FS~\cite{bertinetto2018metalearning}, derived from CIFAR-100 with 64 training, 16 validation, and 20 test classes; (2) tieredImageNet~\cite{bertinetto2018metalearning}, a hierarchical subset of ImageNet comprising 351/97/160 classes across splits; and (3) CUB-200-2011~\cite{triantafillou2019meta}, featuring fine-grained bird classification with 140/30/30 splits. In our paper, we use the testing splits only. 
\end{paragraphs}

\begin{paragraphs}{Graph Datasets:}
We evaluate our approach on two standard citation networks: Cora~\cite{sen2008collective} and CiteSeer~\cite{giles1998citeseer}, where each publication is represented as a node. The network connects publications through citation edges while preserving content features derived from titles and abstracts.
\end{paragraphs}

\begin{table}[!h]
\centering
\setlength{\tabcolsep}{3pt}
\caption{Statistics of Image Datasets.}
\label{tab:citation_nets}
\begin{tabular}{lccc}
\toprule
Dataset & \#Images & \#Classes  & Size \\
\midrule
CIFAR-FS~\cite{bertinetto2018metalearning} & 60,000 &  100 &  32$\times$32  \\
Tiered-ImageNet~\cite{ren2019incremental} & 779,165 & 608 & 224$\times$224  \\
CUB200-2011~\cite{triantafillou2019meta} & 11,745 & 200 & 224$\times$224 \\
\bottomrule
\end{tabular}
\end{table}

\begin{table}[!h]
\centering
\setlength{\tabcolsep}{8.3pt}
\caption{Statistics of Graph Datasets.}
\label{tab:citation_nets}
\begin{tabular}{lccc}
\toprule
Dataset & \#Nodes & \#Edges & \#Classes \\
\midrule
Cora~\cite{sen2008collective} & 2,708 & 5,429 & 7 \\
CiteSeer~\cite{giles1998citeseer} & 3,186 & 4,277 & 6  \\
\bottomrule
\end{tabular}
\end{table}

\subsubsection{Baseline Certification Models}

We compare our approach with state-of-the-art deterministic and efficient certified defenses that are suitable for few-shot scenarios, including DPA~\cite{levine2021deep}, KNN~\cite{jia2022certified}, and Fcert~\cite{wang2024fcert}. The baselines are detailed in Section~\ref{sec:certified_baselines}. Note that, when $\lambda=0$, our LeFCert degrades to FCert. 

\subsubsection{Evaluation Metrics}
Following prior works~\cite{cohen2019certified,lecuyer2019certified}, we evaluate \textit{certified accuracy} as the performance of provable defense. Specifically, it quantifies the percentage of query inputs that are both correctly classified and certifiably robust under an arbitrary poisoning attack within the budget $T$. Certified accuracy is defined as:
\begin{equation}
     \frac{\sum_{i=1}^N \mathbb{I}\{\hat{y}_{test}^{(i)}=y_{test}^{(i)}\}\cdot \mathbb{I}\{\hat{y}_{test}^{(i)} \text{ is certified}\}}{N},
\end{equation}
where $\hat{y}_{test}^{(i)}=g(x_{test}^{(i)};\mathcal{D}^{(i)})$, $\mathcal{D}^{(i)}$ is the support set corresponding to the testing input $x_{test}^{(i)}$, and $y_{test}^{(i)}$ is the label. 

\subsubsection{Few-shot Evaluation Protocols}
\
\newline
\begin{paragraphs}{Default Protocol:}
 Following~\cite{wang2024fcert}, we randomly sample $C$ novel classes from the test set, with $K$ support examples per class. From the remaining examples, we sample one query instance per class to form each episode. This process repeats for $10$ independent episodes (with resampled classes and instances each time) to compute robust certification metrics. The final certified accuracy averages results across all episodes and classes, ensuring comprehensive evaluation of few-shot robustness. We set $C=5, K=10$ by default. 
 \end{paragraphs}
 
 \begin{paragraphs}{Collective Protocol:} In practice, the attacker aims to disrupt the prediction of testing samples as many as possible using a shared budget. To evaluate the effectiveness of collective certification, we evaluate one episode, but sample $100$ query instances per class. Collective certified robustness guarantees the ratio of testing inputs that are certifiably robust under a shared budget $T$.
 \end{paragraphs}

\subsubsection{Model and Parameter  Settings}
\
\newline
\begin{paragraphs}{Image classification:} We consider  CLIP~\cite{radford2021learning} as the language-empowered foundation model. 
By default, we employ cosine distance for LeFCert and $\lambda=25$. 
For the Lipschitz continuity enhanced variants LeFCert-L and LeFCert-LD, we employ $l_2$-norm distance and $\lambda=0.4$. We set $r=0.1$, and $\sigma=1.0$. To obtain the smoothed embedding, we sample 1,000 Gaussian noise to average the feature embeddings. Following~\cite{carlini2023certified}, we employ an off-the-shelf diffusion model. For both CUB-200-2011 and tieredImageNet, we directly apply the pre-trained $256\times256$ class-unconditional diffusion model from Guided Diffusion~\cite{dhariwal2021diffusion}, which was originally trained on ImageNet. Because there is no well-trained diffusion model for CIFAR-FS, we train an unconditional diffusion model on CIFAR-FS using the exact settings of Improved Diffusion~\cite{nichol2021improved} with a $32\times32$ resolution model with $4,000$-step cosine noise schedule as they used for CIFAR10. 
\end{paragraphs}

\begin{paragraphs}{Node classification:} 
We employ GraphCLIP~\cite{zhu2024graphclip} as our language-empowered foundation model. By default, we employ cosine distance for LeFCert and $\lambda=0.7$. For each node, we take the $3$-hops ego-net as its input. For structural context extraction, we implement localized subgraph sampling via 32-step random walks centered on target nodes. Consistent with Y. Zhu~\cite{zhu2024graphclip}'s configurations, our architecture comprises $12$ transformer layers with $384$-dimensional hidden states and utilizes the tiny sentence-transformers variant (all-MiniLM-L6-v2) for efficient text encoding.
\end{paragraphs}

\subsection{Experiment Results}
\begin{table}[!h]
\centering
\setlength{\tabcolsep}{2pt}
\caption{Certified accuracy comparison on images.}
\label{tab:cert_images}
\begin{tabular}{clrrrrrr}
\hline
\multicolumn{2}{c}{5-way, 10-shot} & \multicolumn{1}{c}{Clean} & \multicolumn{5}{c}{Certified Accuracy ($T$)} \\ \hline
Datasets & \multicolumn{1}{l}{Models} & \multicolumn{1}{r}{0} & \multicolumn{1}{r}{1} & \multicolumn{1}{r}{3} & \multicolumn{1}{r}{5} & \multicolumn{1}{r}{7} & \multicolumn{1}{r}{9} \\ \hline
\multirow{6}{*}{CIFAR-FS} & KNN & 0.80 & 0.76 & 0.48 & 0.00 & 0.00 & 0.00 \\
 & DPA & 0.88 & 0.86 & 0.60 & 0.00 & 0.00 & 0.00 \\
 & FCert & 0.88 & 0.84 & 0.72 & 0.00 & 0.00 & 0.00 \\ \cdashline{2-8}[1pt/1pt]
 & \textbf{LeFCert} & \textbf{0.98} & \textbf{0.98} & \textbf{0.96} & 0.00 & 0.00 & 0.00 \\
 & \textbf{LeFCert-L} & 0.76 & 0.70 & 0.44 & 0.06 & 0.00 & 0.00 \\
 & \textbf{LeFCert-LD} & 0.68 & 0.52 & 0.38 & \textbf{0.24} & \textbf{0.06} & \textbf{0.02} \\ \hline
\multirow{6}{*}{\begin{tabular}[c]{@{}c@{}}Tiered-\\ ImageNet\end{tabular}} & KNN & 0.82 & 0.76 & 0.56 & 0.00 & 0.00 & 0.00 \\
 & DPA & 0.88 & 0.82 & 0.66 & 0.00 & 0.00 & 0.00 \\
 & FCert & 0.92 & 0.88 & 0.70 & 0.00 & 0.00 & 0.00 \\ \cdashline{2-8}[1pt/1pt] 
 & \textbf{LeFCert} & \textbf{0.98} & \textbf{0.96} & \textbf{0.92} & 0.00 & 0.00 & 0.00 \\
 & \textbf{LeFCert-L} & 0.84 & 0.66 & 0.42 & 0.20 & 0.08 & 0.04 \\
 & \textbf{LeFCert-LD} & 0.96 & 0.94 & 0.74 & \textbf{0.62} & \textbf{0.48} & \textbf{0.28} \\ \hline
\multirow{6}{*}{\begin{tabular}[c]{@{}c@{}}CUB200-\\ 2011\end{tabular}} & KNN & 0.86 & 0.78 & 0.46 & 0.00 & 0.00 & 0.00 \\
 & DPA & 0.94 & 0.92 & 0.66 & 0.00 & 0.00 & 0.00 \\
 & FCert & 0.96 & 0.88 & 0.66 & 0.00 & 0.00 & 0.00 \\ \cdashline{2-8}[1pt/1pt] 
 & \textbf{LeFCert} & \textbf{0.96} & \textbf{0.94} & \textbf{0.84} & 0.00 & 0.00 & 0.00 \\
 & \textbf{LeFCert-L} & 0.80 & 0.70 & 0.48 & 0.26 & 0.14 & 0.08 \\
 & \textbf{LeFCert-LD} & 0.90 & 0.80 & 0.70 & \textbf{0.54} & \textbf{0.28} & \textbf{0.22} \\ \hline
\end{tabular}
\end{table}

\begin{figure}[!ht]
    \centering
    \subfigure[CIFAR-FS]{\includegraphics[width=0.235\textwidth,height=2.95cm]{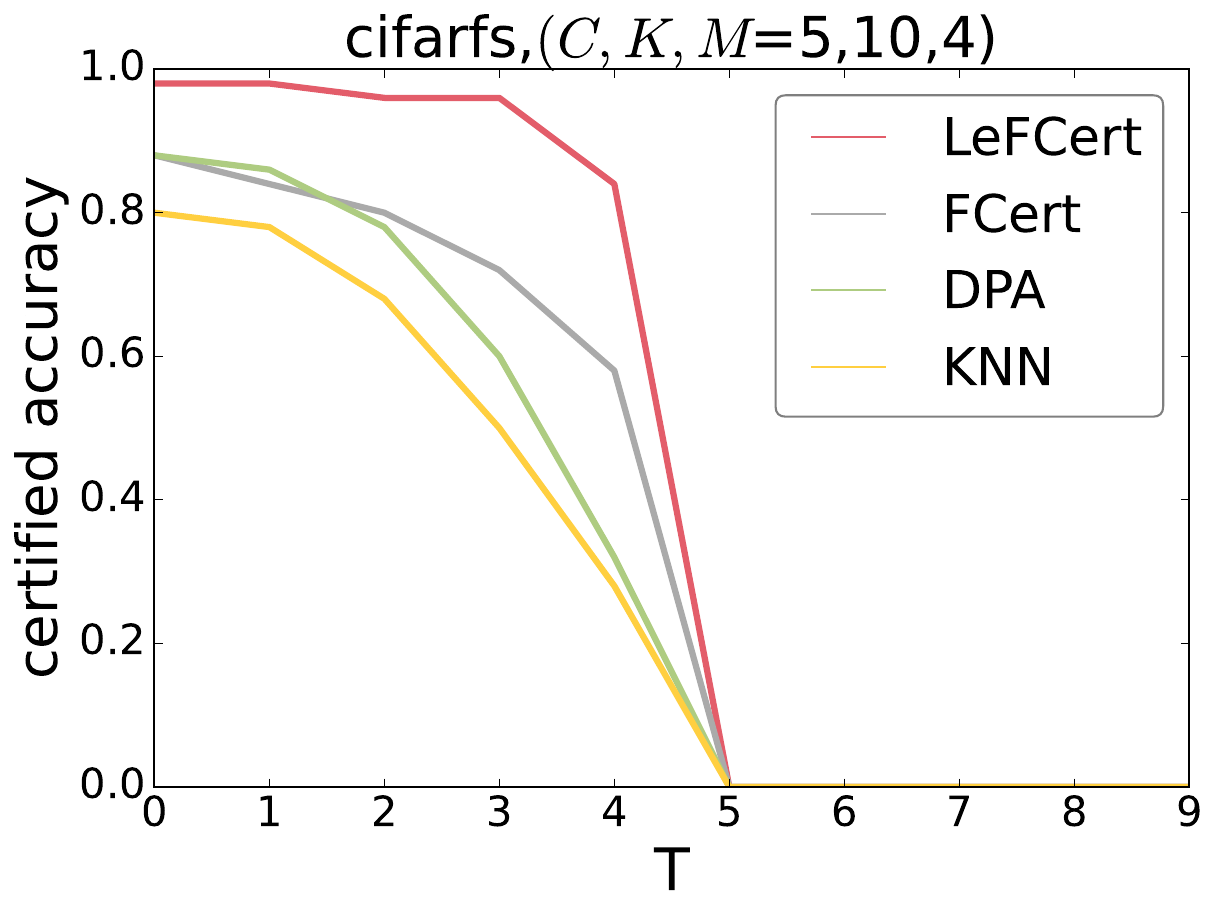}}
    \subfigure[CUB200-2011]{\includegraphics[width=0.235\textwidth,height=2.95cm]{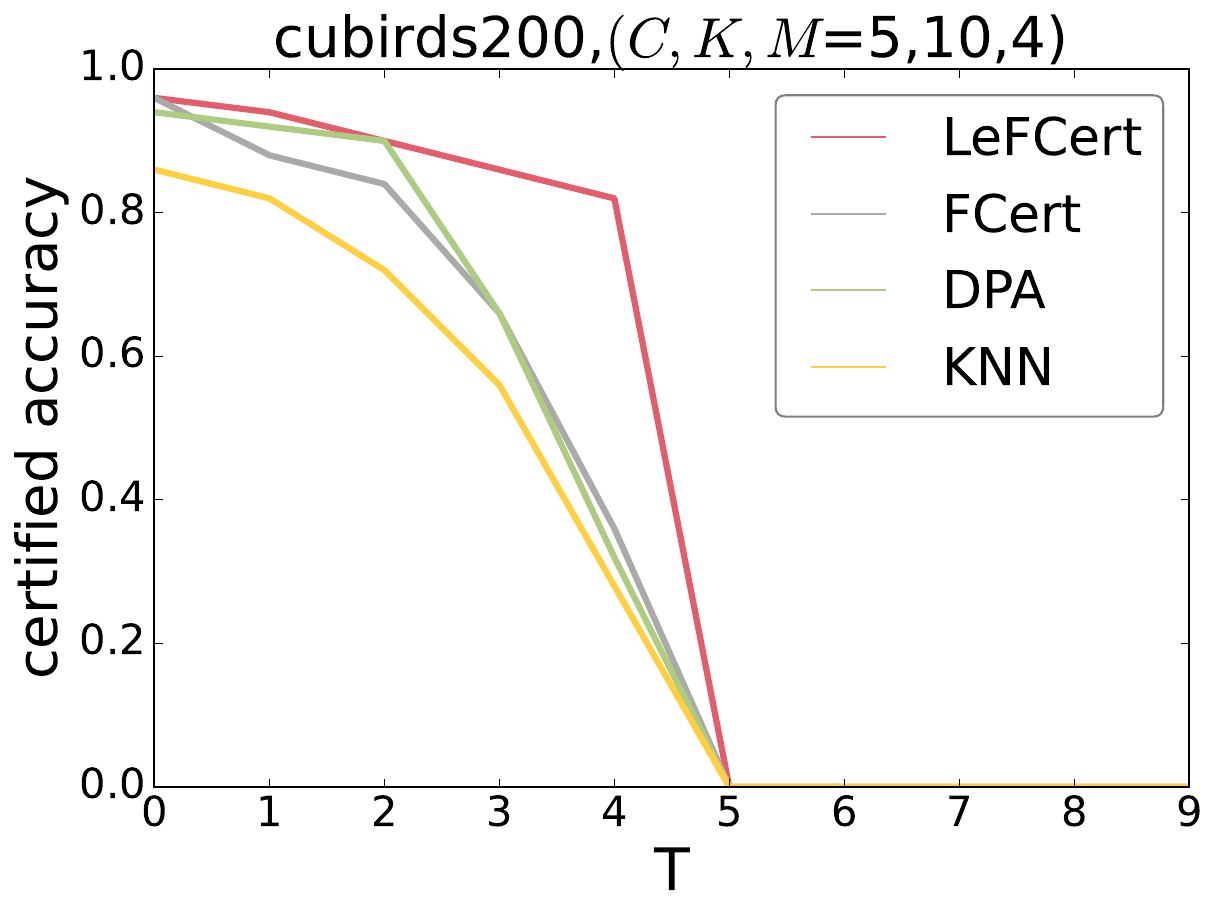}}
    \subfigure[Tiered-ImageNet]{\includegraphics[width=0.235\textwidth,height=2.95cm]{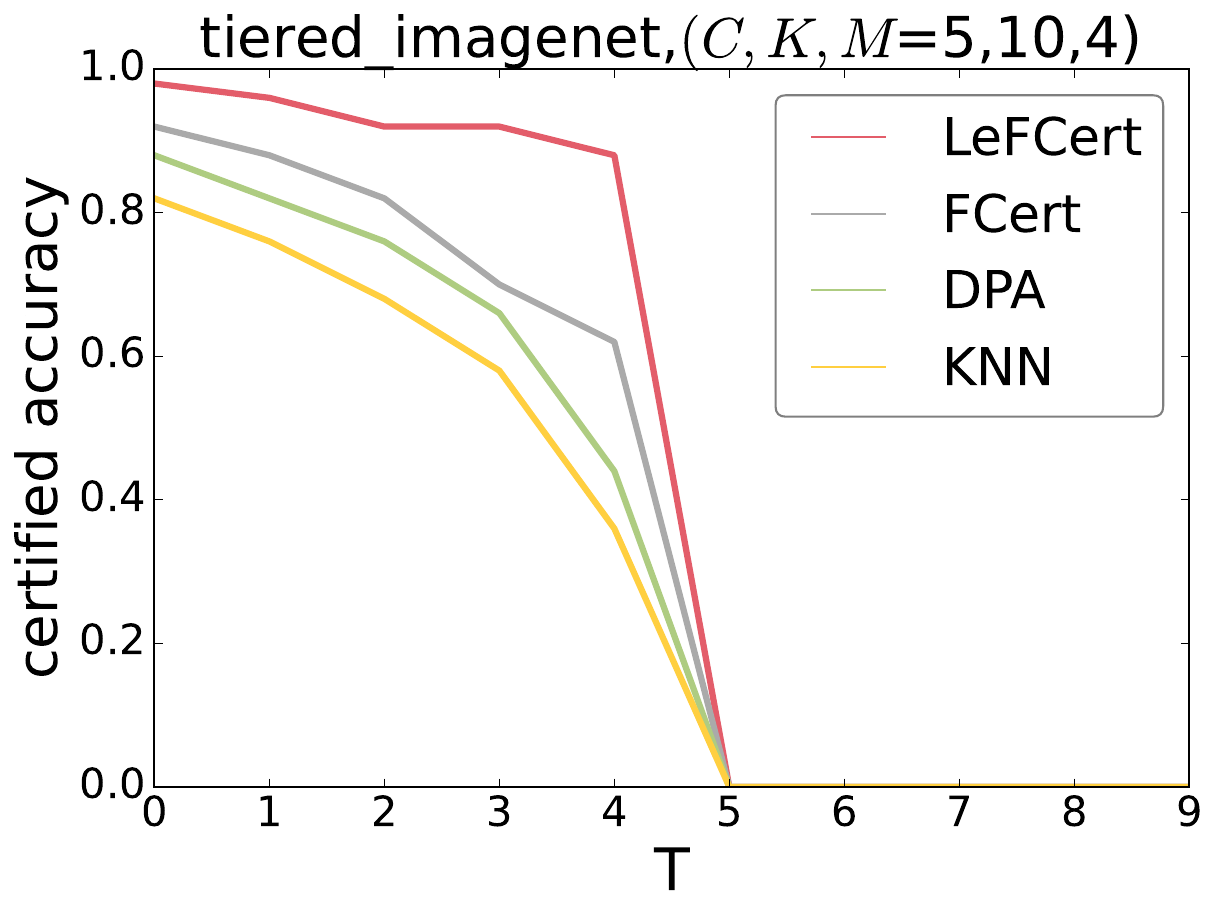}}
    \subfigure[Tiered-ImageNet]{\includegraphics[width=0.235\textwidth,height=2.95cm]{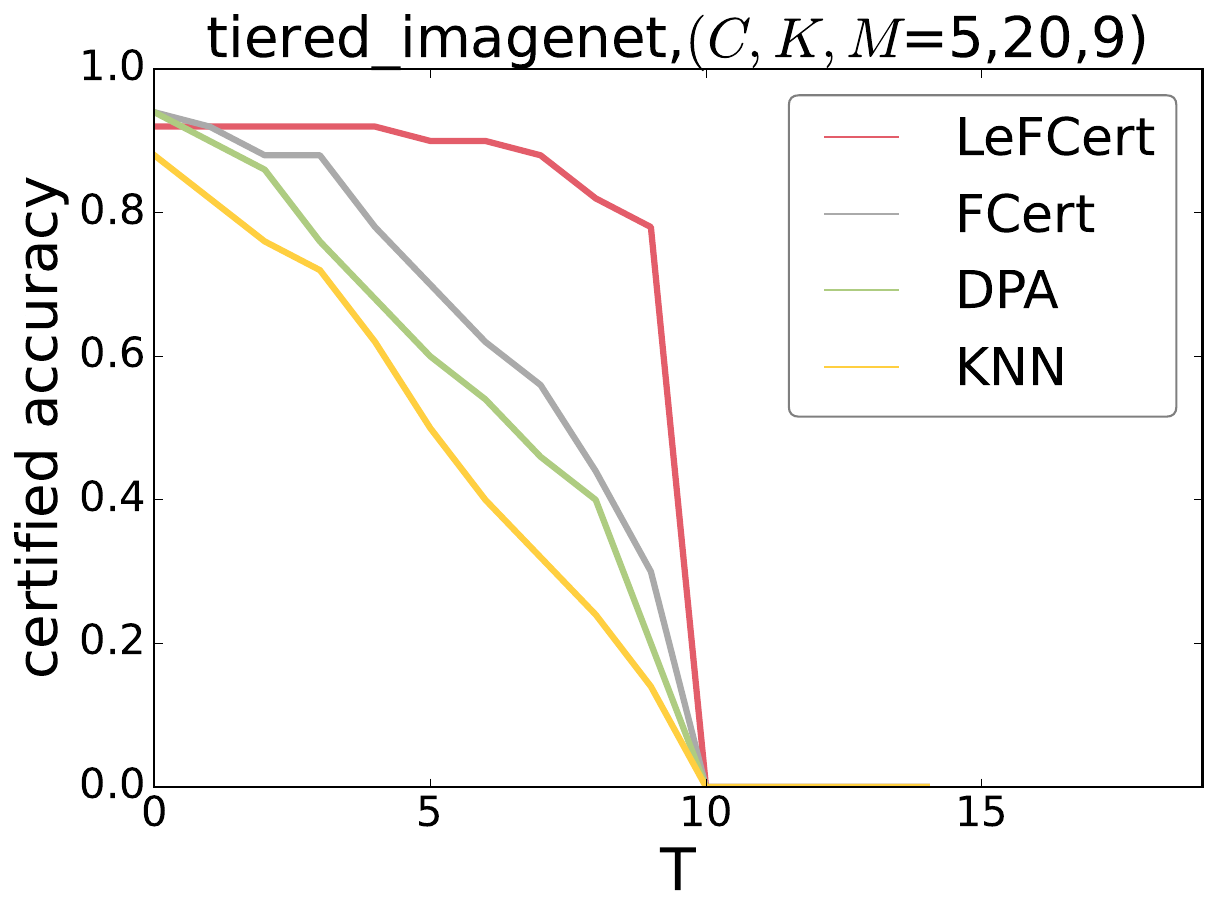}}
    \caption{Certified accuracy comparison on images. \textbf{LeFCert (red) outperforms the baseline significantly.}}
\label{fig:images_Lipt_sigma1.0}
\vspace{-5pt}
\end{figure}

\subsubsection{Certified Accuracy Comparison}





To address \textbf{Q1}, we evaluate the certified robustness of LeFCert and its variants (LeFCert-L, LeFCert-LD, and LeFCert-C) against state-of-the-art baselines, including KNN, DPA, and FCert. The certified accuracy results are presented in Tables~\ref{tab:cert_images}, \ref{tab:cert_graphs}, \ref{tab:cert_collect}, and Figures~\ref{fig:images_Lipt_sigma1.0} and~\ref{fig:graphs}. These results demonstrate that LeFCert and its variants consistently outperform state-of-the-art baselines, providing substantial improvements in certified robustness under various poisoning scenarios.

On image datasets, LeFCert consistently achieves significant improvements in certified accuracy over the baselines (Table~\ref{tab:cert_images}, Figure~\ref{fig:images_Lipt_sigma1.0}). For instance, at a poisoning size of $T=3$, on CIFAR-FS, LeFCert achieves a certified accuracy of 96\%, which outperforms FCert, DPA, and KNN by 33\%, 60\%, and 100\%, respectively. On Tiered-ImageNet, LeFCert achieves a certified accuracy of 92\%, representing a 31\% improvement over FCert and a 39\% improvement over DPA. Similarly, on CUB200-2011, LeFCert achieves 84\% certified accuracy, while FCert and DPA are 66\%, and KNN is 46\%. Furthermore, LeFCert possesses the highest clean accuracy across all the datasets. These results validate the effectiveness of incorporating textual embeddings, which enhance the model’s ability to counteract adversarial poisoning.

\begin{figure}[!ht]
    \centering
    \subfigure[Cora]{\includegraphics[width=0.235\textwidth,height=2.95cm]{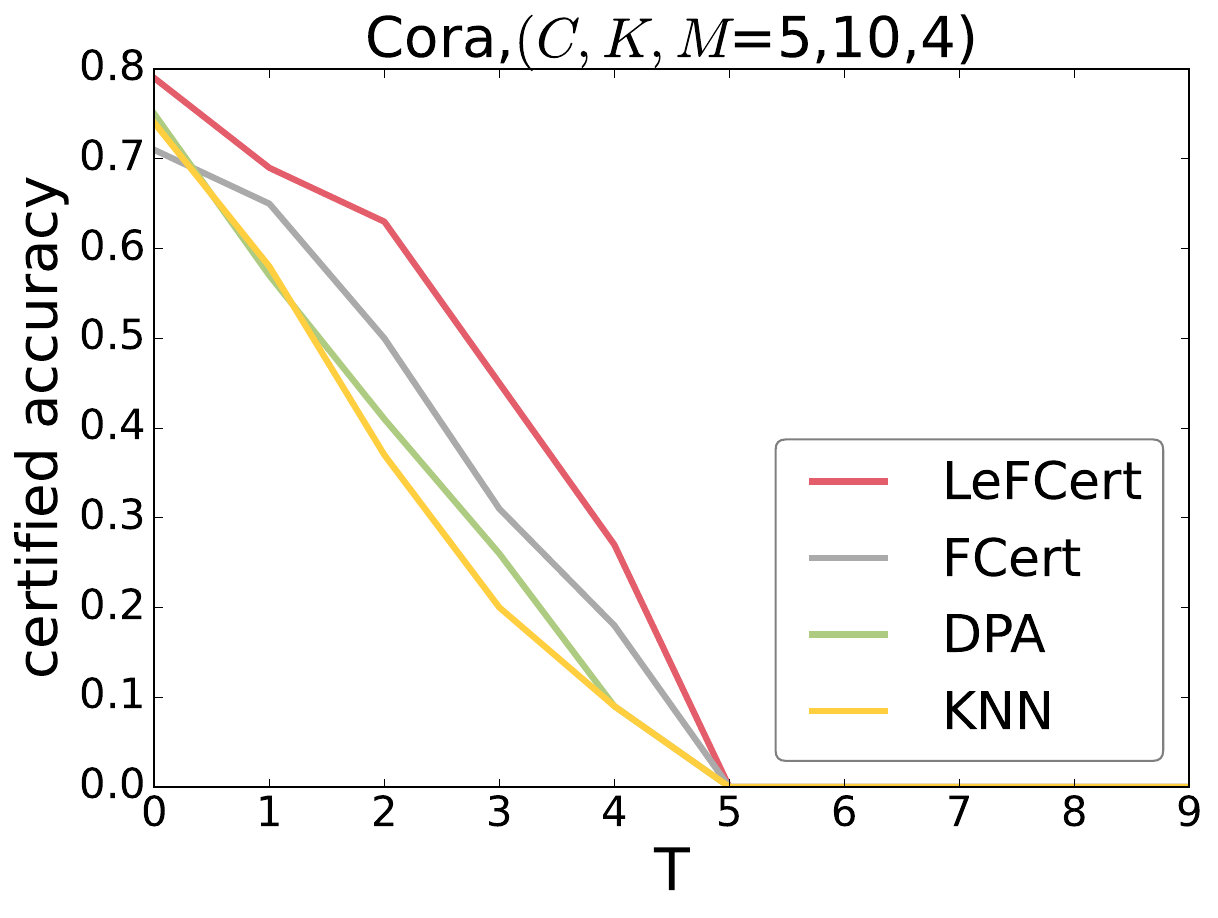}}
    \subfigure[Citeseer]{\includegraphics[width=0.235\textwidth,height=2.95cm]{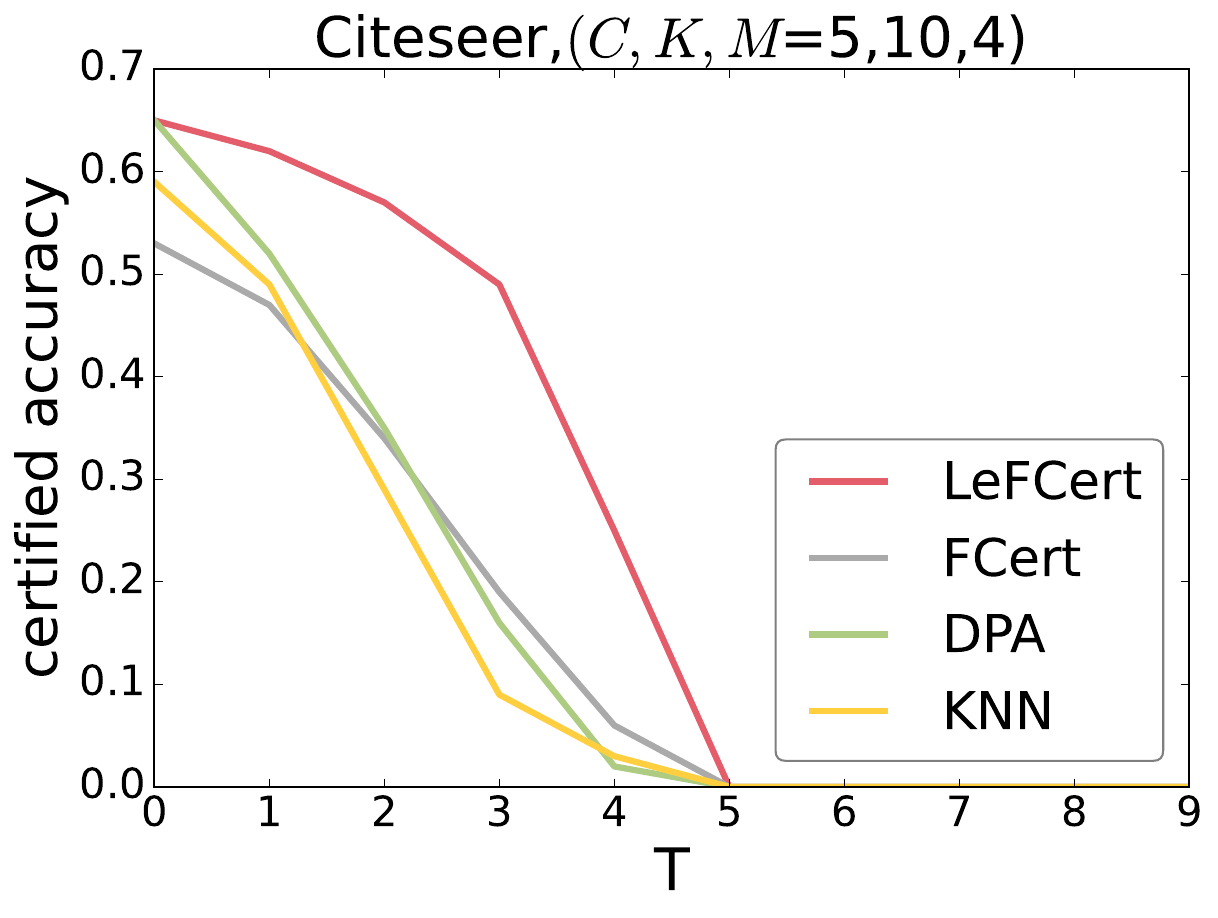}}
    \caption{Certified accuracy comparison on graphs.}
    \label{fig:graphs}
\end{figure}

\begin{table}[!h]
\centering
\setlength{\tabcolsep}{2pt}
\caption{Certified accuracy comparison on graphs.}
\label{tab:cert_graphs}
\begin{tabular}{clrrrrrr}
\hline
\multicolumn{2}{c}{5-way, 10-shot} & \multicolumn{1}{c}{Clean} & \multicolumn{5}{c}{Certified Accuracy ($T$)} \\ \hline
Datasets & \multicolumn{1}{l}{Models} & \multicolumn{1}{r}{0} & \multicolumn{1}{r}{1} & \multicolumn{1}{r}{2} & \multicolumn{1}{r}{3} & \multicolumn{1}{r}{4} & \multicolumn{1}{r}{5} \\ \hline
\multirow{4}{*}{Cora} & KNN & 0.74 & 0.58 & 0.37 & 0.20 & 0.09 & 0.00 \\
 & DPA & 0.75 & 0.57 & 0.41 & 0.26 & 0.09 & 0.00 \\
 & FCert & 0.71 & 0.65 & 0.50 & 0.31 & 0.18 & 0.00 \\ \cdashline{2-8}[1pt/1pt]
 & \textbf{LeFCert} & \textbf{0.79} & \textbf{0.69} & \textbf{0.63} & \textbf{0.45} & \textbf{0.27} & 0.00 \\ \hline
\multirow{4}{*}{CiteSeer} & KNN & 0.59 & 0.49 & 0.29 & 0.09 & 0.03 & 0.00 \\
 & DPA & 0.65 & 0.52 & 0.35 & 0.16 & 0.02 & 0.00 \\
 & FCert & 0.53 & 0.47 & 0.34 & 0.19 & 0.06 & 0.00 \\ \cdashline{2-8}[1pt/1pt]
 & \textbf{LeFCert} & \textbf{0.65} & \textbf{0.62} & \textbf{0.57} & \textbf{0.49} & \textbf{0.25} & 0.00 \\ \hline
\end{tabular}
\end{table}

On graph datasets, LeFCert also shows significantly superior certified robustness (Figure~\ref{fig:graphs}). On CiteSeer, at $T=3$, LeFCert achieves a certified accuracy of 49\%, outperforming FCert, DPA, and KNN by 157\%, 206\%, and 444\%, respectively. A similar advantage is observed on the Cora dataset. These results highlight the generalization capability and broad applicability of LeFCert, making it a robust choice across diverse tasks.

The variants LeFCert-L and LeFCert-LD demonstrate remarkable robustness under high poisoning budgets (Tables~\ref{tab:cert_images}). For example, on Tiered-ImageNet with $T=7$, LeFCert-LD achieves a certified accuracy of 48\%, while all the baselines are 0\%. Comparing LeFCert-LD with LeFCert-L, the denoise diffusion model can effectively rescue the clean accuracy drop caused by the random noise. Obtaining Lipschitz continuity through randomized smoothing (LeFCert-L) and further denoising input perturbations through a diffusion model (LeFCert-LD), the experiment results illustrate that this design is particularly effective in dual constraint scenarios where the adversary can manipulate the images within an imperceptible $l_2$-norm ball.

\begin{table}[!ht]
\centering
\setlength{\tabcolsep}{2pt}
\caption{Certified accuracy comparison on images and graphs (Sample-wise v.s. Collective).}
\label{tab:cert_collect}
\begin{tabular}{clrrrrrr}
\hline
\multicolumn{2}{c}{5-way, 10-shot} & \multicolumn{1}{l}{Clean} & \multicolumn{5}{c}{Certified Accuracy (T)} \\ \hline
\multicolumn{1}{l}{Datasets} & Models & 0 & 1 & 2 & 3 & 4 & 5 \\ \hline
\multirow{3}{*}{CIFAR-FS} & FCert & 0.85 & 0.80 & 0.72 & 0.61 & 0.37 & 0.00 \\
 & \textbf{LeFCert} & \textbf{0.99} & \textbf{0.99} & \textbf{0.99} & 0.96 & 0.90 & 0.00 \\
 & \textbf{LeFCert-C} & \textbf{0.99} & \textbf{0.99} & \textbf{0.99} & \textbf{0.97} & \textbf{0.94} & 0.00 \\ \hline
\multirow{3}{*}{\begin{tabular}[c]{@{}c@{}}Tiered-\\ ImageNet\end{tabular}} & FCert & 0.88 & 0.82 & 0.76 & 0.64 & 0.49 & 0.00 \\
 & \textbf{LeFCert} & \textbf{0.99} & 0.96 & 0.93 & 0.85 & 0.71 & 0.00 \\
 & \textbf{LeFCert-C} & \textbf{0.99} & \textbf{0.97} & \textbf{0.95} & \textbf{0.91} & \textbf{0.85} & 0.00 \\ \hline
\multirow{3}{*}{Cora} & FCert & 0.71 & 0.61 & 0.46 & 0.31 & 0.17 & 0.00 \\
 & \textbf{LeFCert} & \textbf{0.72} & 0.61 & 0.50 & 0.37 & 0.24 & 0.00 \\
 & \textbf{LeFCert-C} & \textbf{0.72} & \textbf{0.67} & \textbf{0.58} & \textbf{0.46} & \textbf{0.41} & 0.00 \\ \hline
\multirow{3}{*}{CiteSeer} & FCert & 0.59 & 0.49 & 0.36 & 0.22 & 0.07 & 0.00 \\
 & \textbf{LeFCert} & \textbf{0.66} & 0.58 & 0.51 & 0.41 & 0.15 & 0.00 \\
 & \textbf{LeFCert-C} & \textbf{0.66} & \textbf{0.62} & \textbf{0.59} & \textbf{0.55} & \textbf{0.31} & 0.00 \\ \hline
\end{tabular}
\end{table}

LeFCert-C extends certified robustness by addressing a collective budget constraint, where an attacker seeks to disrupt as many test samples as possible within a shared poisoning budget. As shown in Table~\ref{tab:cert_collect}, on CIFAR-FS at $T=4$, LeFCert-C achieves a certified accuracy of 94\%, representing a 4\% improvement over LeFCert and a remarkable 154\% improvement over FCert (37\%). Similarly, on Tiered-ImageNet at $T=4$, LeFCert-C achieves 85\%, surpassing LeFCert by 20\% and FCert by 73\%. Even larger improvements are observed on graph datasets, further emphasizing its effectiveness. These results validate the strength of LeFCert-C’s collective certification design, which enhances robustness by effectively modeling tighter attack constraints.

\subsubsection{Accuracy-Robustness Trade-off }

\begin{figure}[!ht]
    \centering
    \subfigure[Cora]{\includegraphics[width=0.235\textwidth,height=2.95cm]{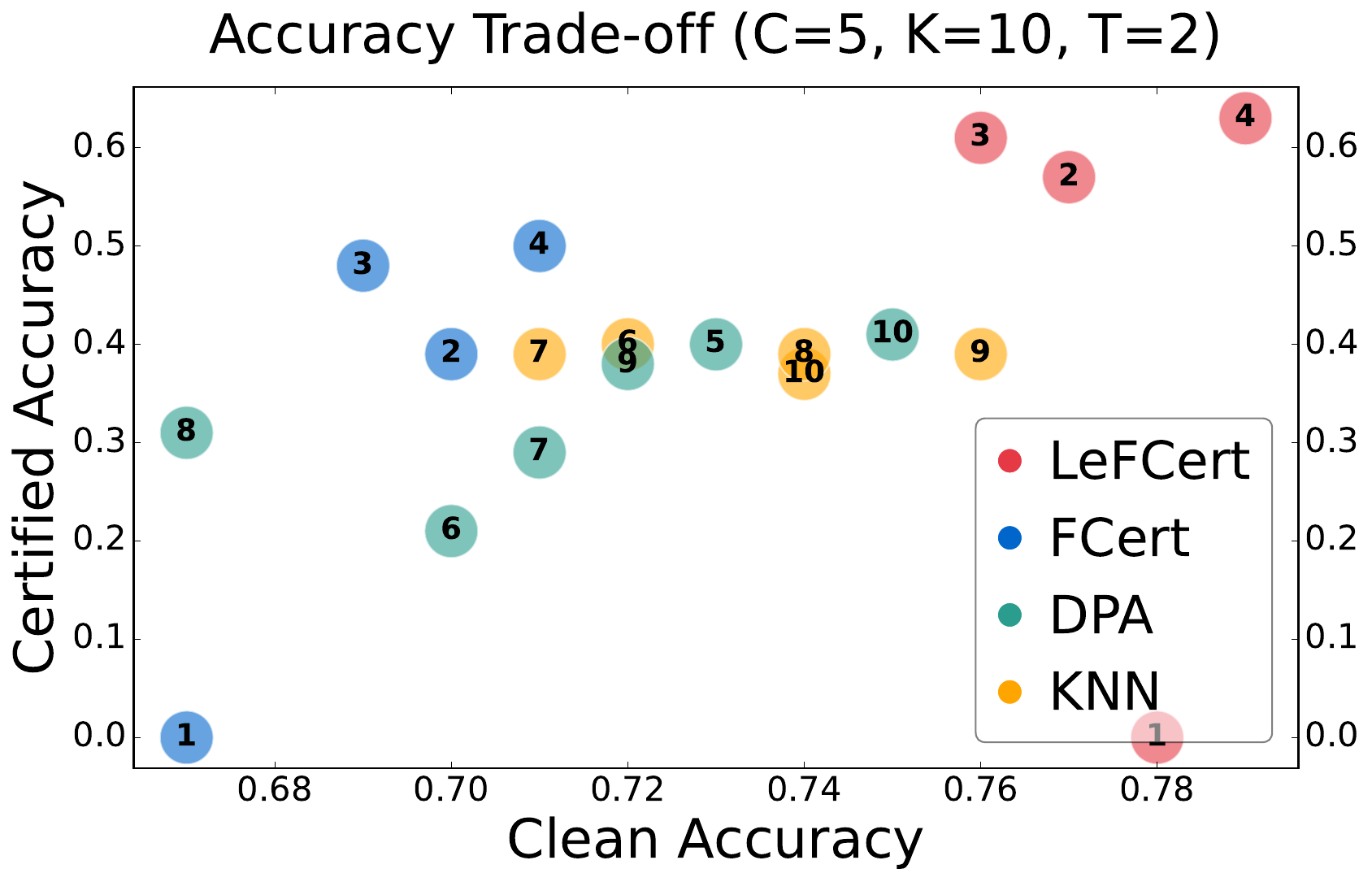}}
    \subfigure[Citeseer]{\includegraphics[width=0.235\textwidth,height=2.95cm]{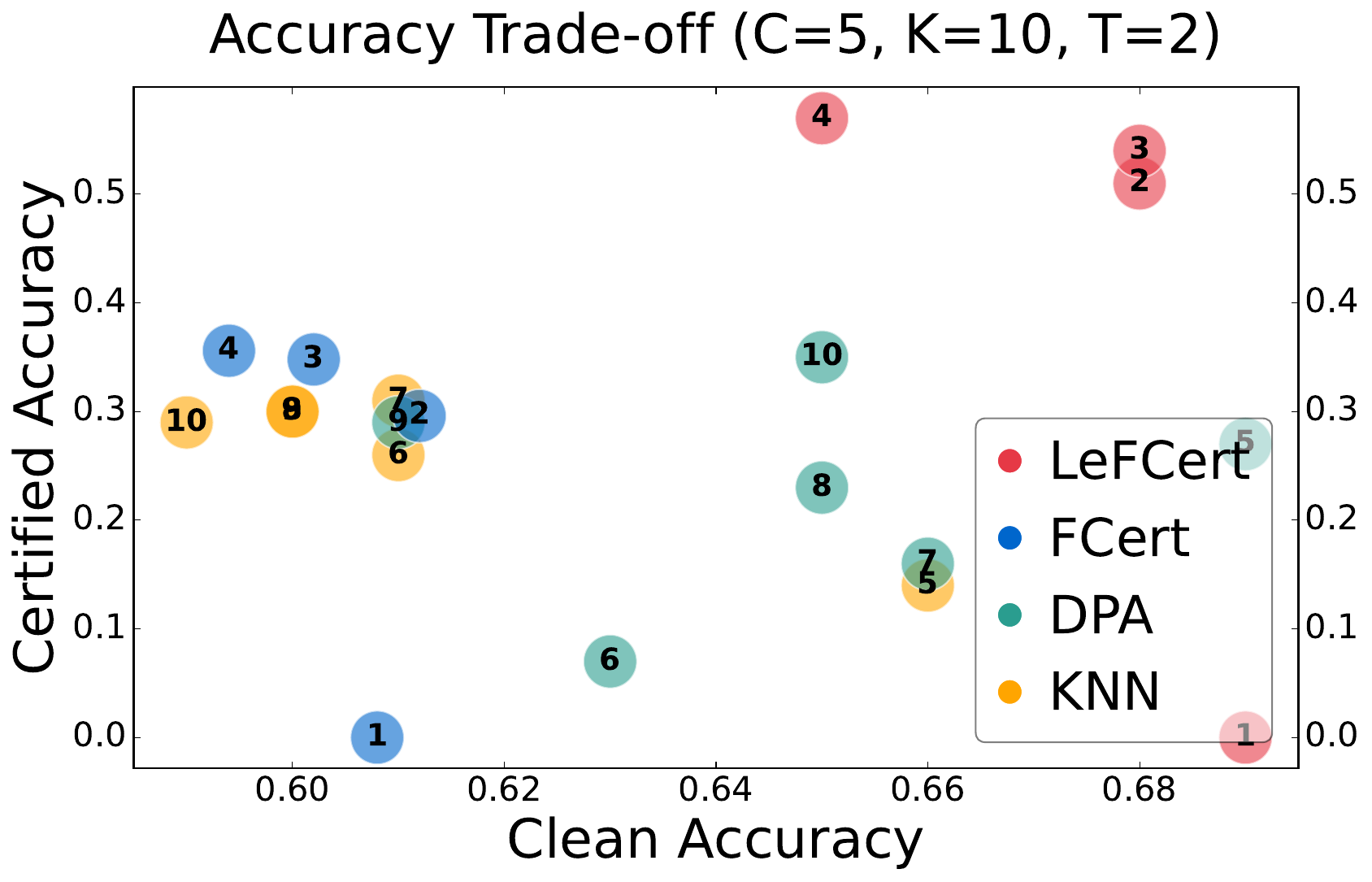}}
    \subfigure[CIFAR-FS]{\includegraphics[width=0.235\textwidth,height=2.95cm]{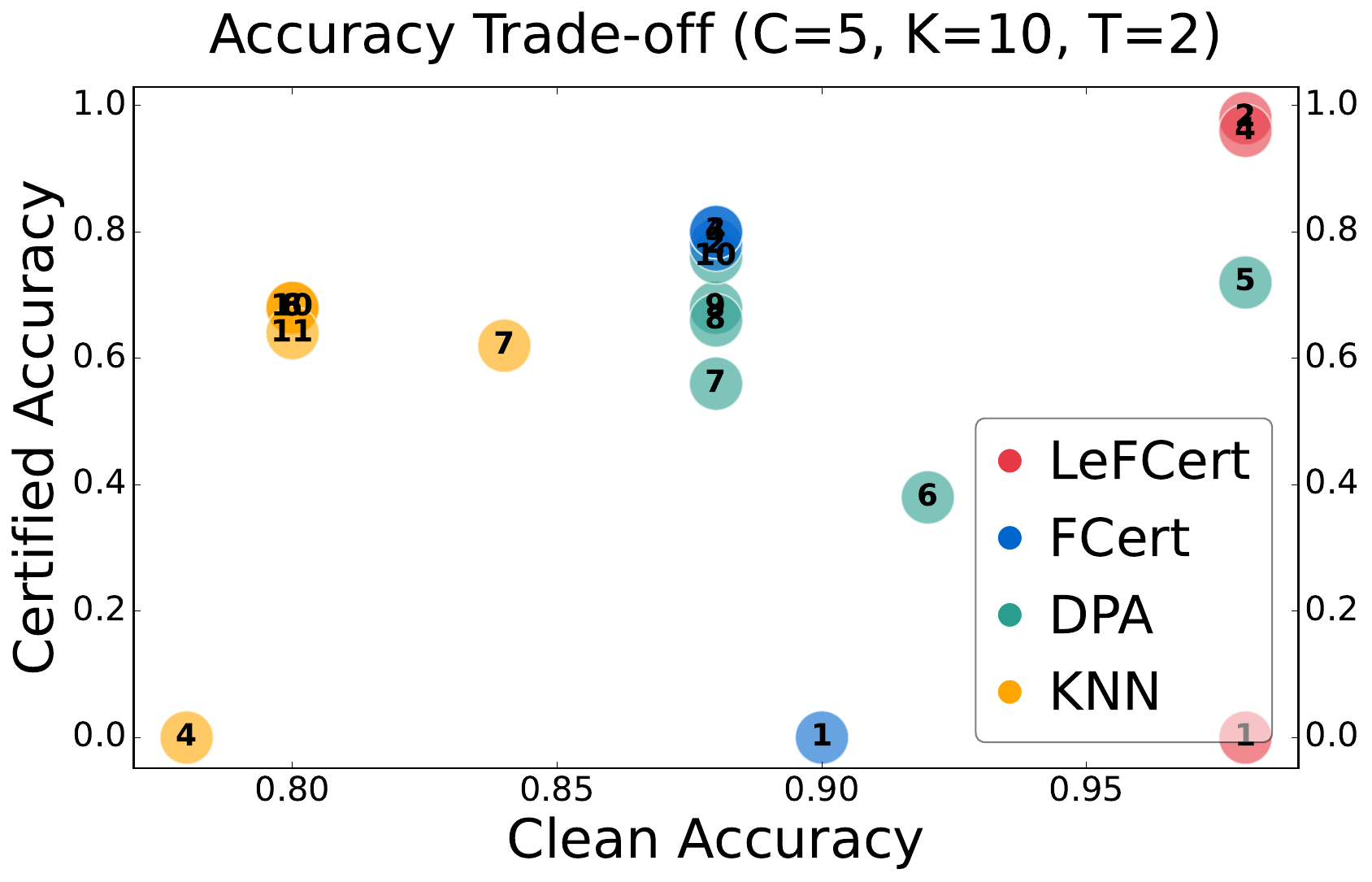}}
    \subfigure[CUB200-2011]{\includegraphics[width=0.235\textwidth,height=2.95cm]{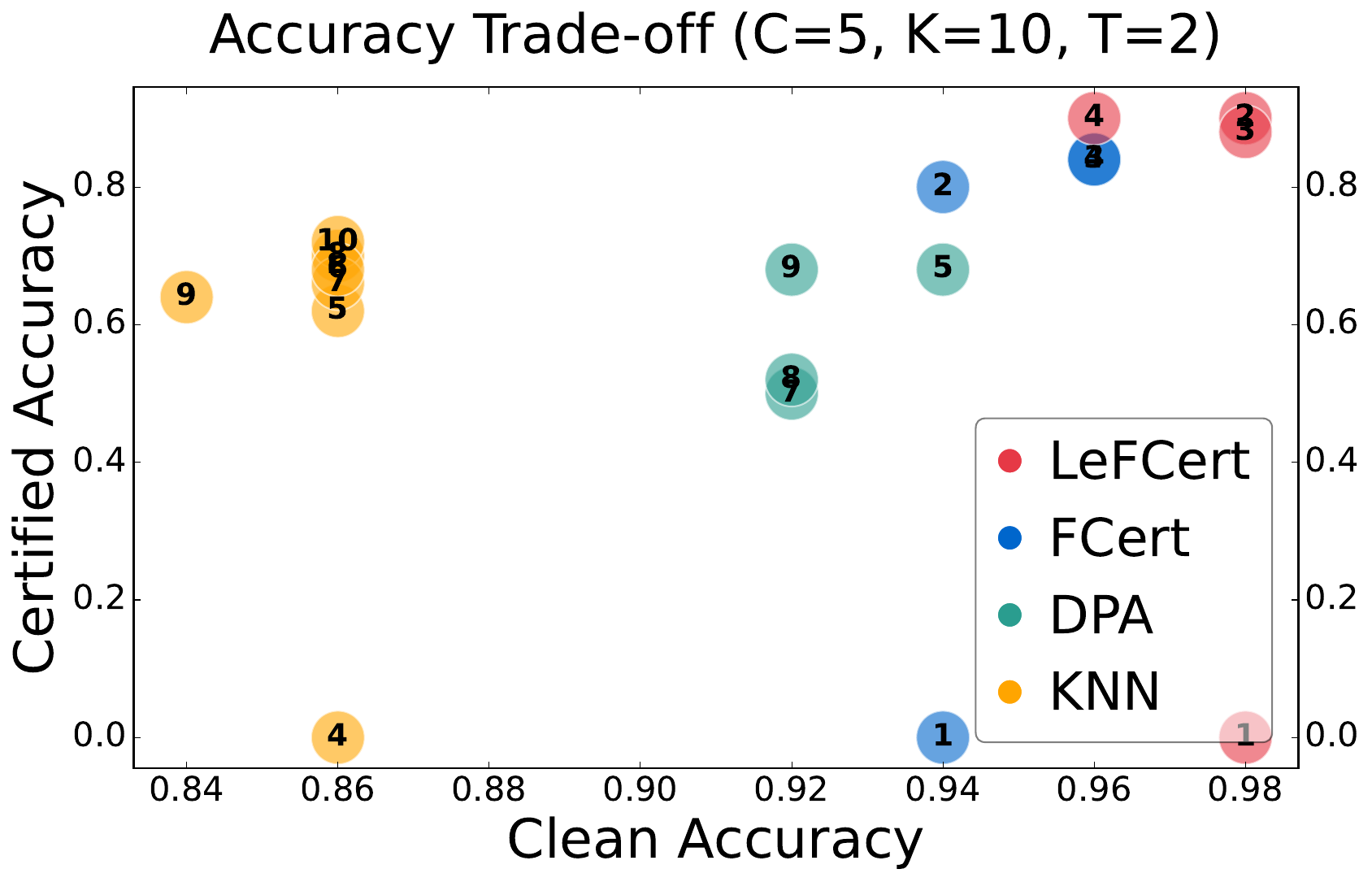}}
    \subfigure[Tiered-ImageNet]{\includegraphics[width=0.235\textwidth,height=2.95cm]{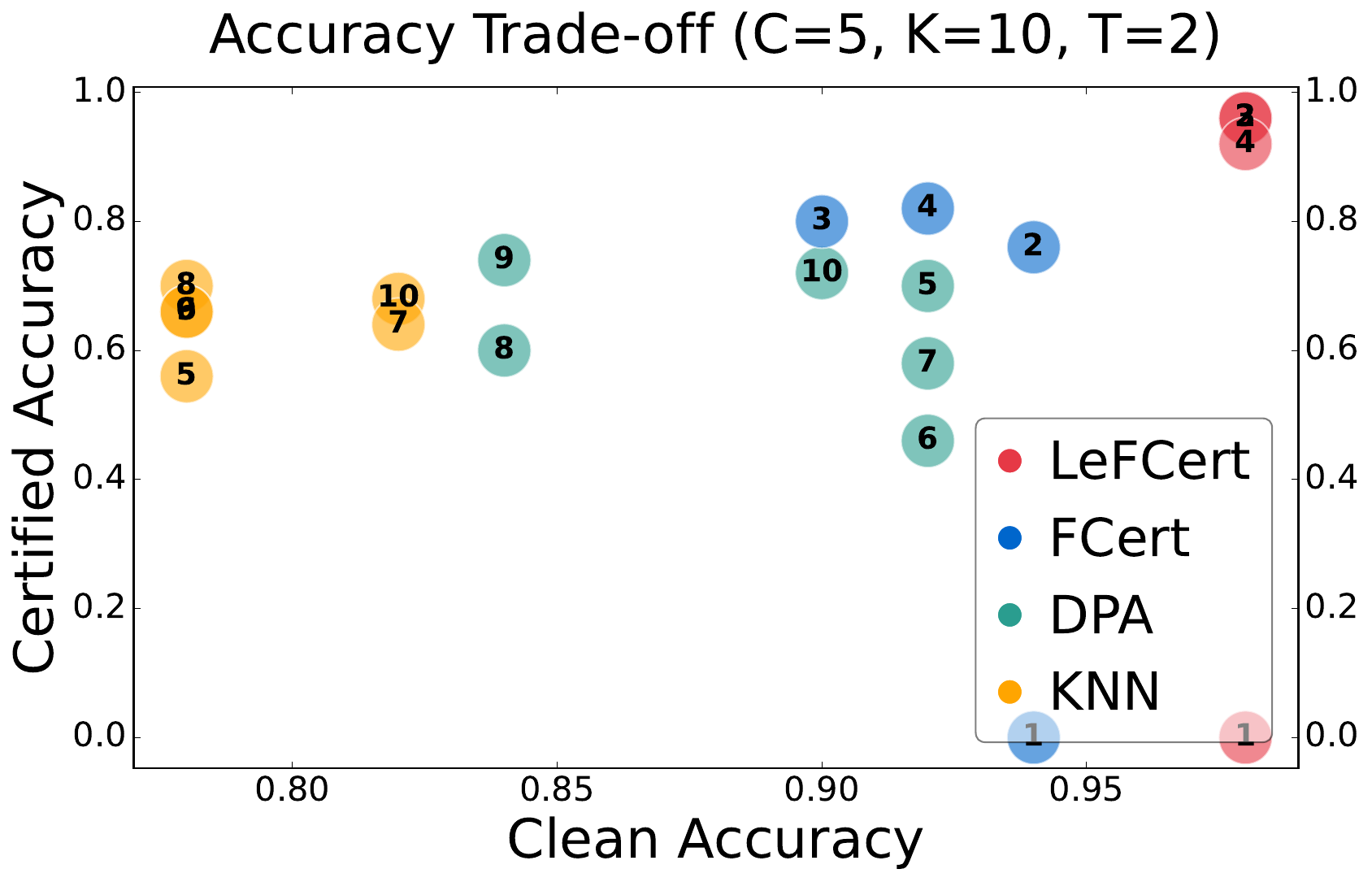}}
    \subfigure[Tiered-ImageNet]{\includegraphics[width=0.235\textwidth,height=2.95cm]{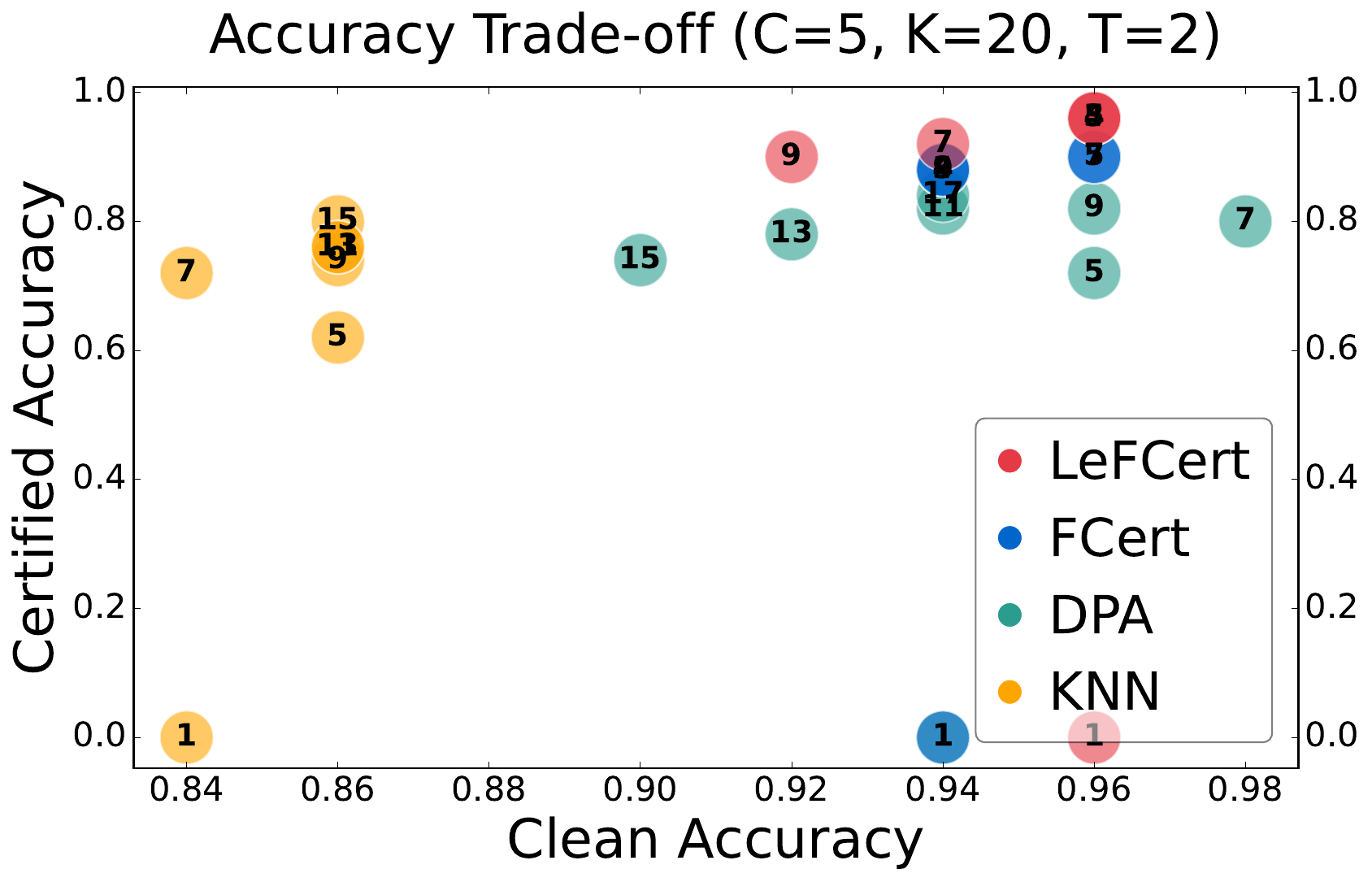}}
    \caption{Accuracy trade-off with various parameters. \textbf{LeFCert (represented by red circles) consistently outperforms the baselines, achieving a significantly better accuracy-robustness trade-off}.}
\label{fig:Tradeoff}
\vspace{-10pt}
\end{figure}

Figure~\ref{fig:Tradeoff} illustrates the trade-off between model (clean) accuracy and certifiably robust accuracy across different models and various parameters. For KNN, this trade-off is governed by the hyperparameter $k$, which determines the number of nearest neighbors used for predictions. A larger $k$ can guarantee more attacked samples, but it might decrease the clean accuracy. Similarly, in DPA, increasing the group number allows the model to tolerate a larger poisoning budget but reduces the availability of support samples per group, thereby negatively impacting clean accuracy. In both FCert and LeFCert, the trade-off is influenced by the trimming parameter $M$, which balances the exclusion of outliers and the retention of support samples. The circles closer to the upper right corner represent both higher clean and certified accuracy. Across all datasets, LeFCert (represented by red circles) consistently outperforms the baselines, achieving a significantly better accuracy-robustness trade-off. Unlike the baselines, LeFCert effectively leverages its integration of textual and feature embeddings, along with its robust trimming mechanisms, to mitigate the impact of adversarial poisoning while simultaneously preserving clean accuracy.

\subsubsection{Computation Efficiency}

To address \textbf{Q3}, we benchmark the computational efficiency of LeFCert and its variants against baselines under both default and collective protocols. The results, summarized in Tables~\ref{tab:runtime} and~\ref{tab:runtime_collect}, highlight the superior balance achieved by LeFCert between performance and computational cost.

Under the default protocol (Tables~\ref{tab:runtime}), LeFCert demonstrates a small computational burden while delivering significant improvements in robustness. For instance, on CIFAR-FS, LeFCert requires only 2.13 seconds for 10 episodes, which is only 13\% higher than FCert (1.88 seconds) and 4\% higher than KNN (2.04 seconds). These results demonstrate that despite the integration of textual embeddings and advanced robustness mechanisms, LeFCert maintains computational efficiency comparable to simpler baselines like FCert, while significantly outperforming models like DPA (e.g., 2.13 vs. 3.30 seconds on CIFAR-FS). Similar results are observed in other datasets.

In contrast, the advanced variants LeFCert-L and LeFCert-LD incur higher computational costs due to their reliance on randomized smoothing and diffusion denoise smoothing. For example, on Tiered-ImageNet, LeFCert-L requires 51.30 seconds per episode, while LeFCert-LD requires 3678.23 seconds. These results highlight the trade-off between extreme robustness and computational efficiency, with LeFCert-L and LeFCert-LD being better suited for scenarios requiring robustness under high poisoning size (T).

Under the collective protocol, LeFCert continues to demonstrate its efficiency, with runtime only marginally higher than FCert (Table~\ref{tab:runtime_collect}). For example, on CIFAR-FS, LeFCert takes 2.42 seconds, just 5\% higher than FCert (2.30 seconds). LeFCert-C, designed for collective certification, incurs a higher runtime (e.g., 5.95 seconds on CIFAR-FS) due to its more comprehensive optimization of shared poisoning budgets. In Figure~\ref{fig:time_Collect} (in Appendix~\ref{sec:more_results}), we further visualize the runtime of budget allocation as the budget $T$ increases. Although the runtime increases significantly with larger budgets, it is important to note that in few-shot scenarios, the number of shots $K$ is typically small, resulting in a limited budget range ($T \leq K-M-1$). For a larger $K=20$, it takes less than $25s$. This ensures the practicality of our algorithm. Furthermore, this additional cost is justified by the significant robustness improvements.

\begin{table}[!ht]
\centering
\setlength{\tabcolsep}{2pt}
\caption{Comparison of running time (s) for 10 episodes (Default Protocol). \textbf{Within $0.9s$, LeFCert can verify $C$ testing samples per episode, which is similar to the baselines}.}
\label{tab:runtime}
\begin{tabular}{crrrrr}
\hline
 Models & \multicolumn{1}{c}{\begin{tabular}[c]{@{}c@{}}CIFAR-\\FS\end{tabular}} & \multicolumn{1}{c}{\begin{tabular}[c]{@{}c@{}}Tiered-\\ ImageNet\end{tabular}} & \multicolumn{1}{c}{\begin{tabular}[c]{@{}c@{}}CUB200-\\ 2011\end{tabular}} & \multicolumn{1}{c}{Cora} & \multicolumn{1}{c}{CiteSeer} \\ \hline
KNN & 2.04 & 5.24 & 3.36 & 6.96 & 3.44 \\
DPA & 3.30 & 7.40 & 5.43 & 8.63 & 4.98 \\
FCert & 1.88 & 5.33 & 3.40 & 8.11 & 3.88 \\
\cdashline{1-6}[1pt/1pt]
\textbf{LeFCert} & 2.13 & 5.48 & 4.14 & 8.67 & 4.19 \\
\textbf{LeFCert-L} & 18.71 & 51.30 & 20.14 & - & - \\
\textbf{LeFCert-LD} & 42.51 & 3678.23 & 3716.74 & - & - \\ \hline
\end{tabular}
\end{table}

\begin{table}[!ht]
\centering
\setlength{\tabcolsep}{2pt}
\caption{Comparison of running time (s) (Collective Protocol). \textbf{LeFCert-C can verify $100$ testing samples within $7s$}.}
\label{tab:runtime_collect}
\begin{tabular}{crrrr}
\hline
Models & \multicolumn{1}{c}{CIFAR-FS} & \multicolumn{1}{c}{Tiered-ImageNet} & \multicolumn{1}{c}{Cora} & \multicolumn{1}{c}{CiteSeer} \\ \hline
FCert & 2.30 & 2.51 & 3.17 & 1.90 \\
\textbf{LeFCert} & 2.42 & 2.71 & 3.73 & 2.64 \\
\textbf{LeFCert-C} & 5.95 & 6.38 & 6.34 & 5.10 \\ \hline
\end{tabular}
\end{table}


\begin{figure}[!ht]
    \centering
    \subfigure[CUB200-2011]{\includegraphics[width=0.235\textwidth,height=2.95cm]{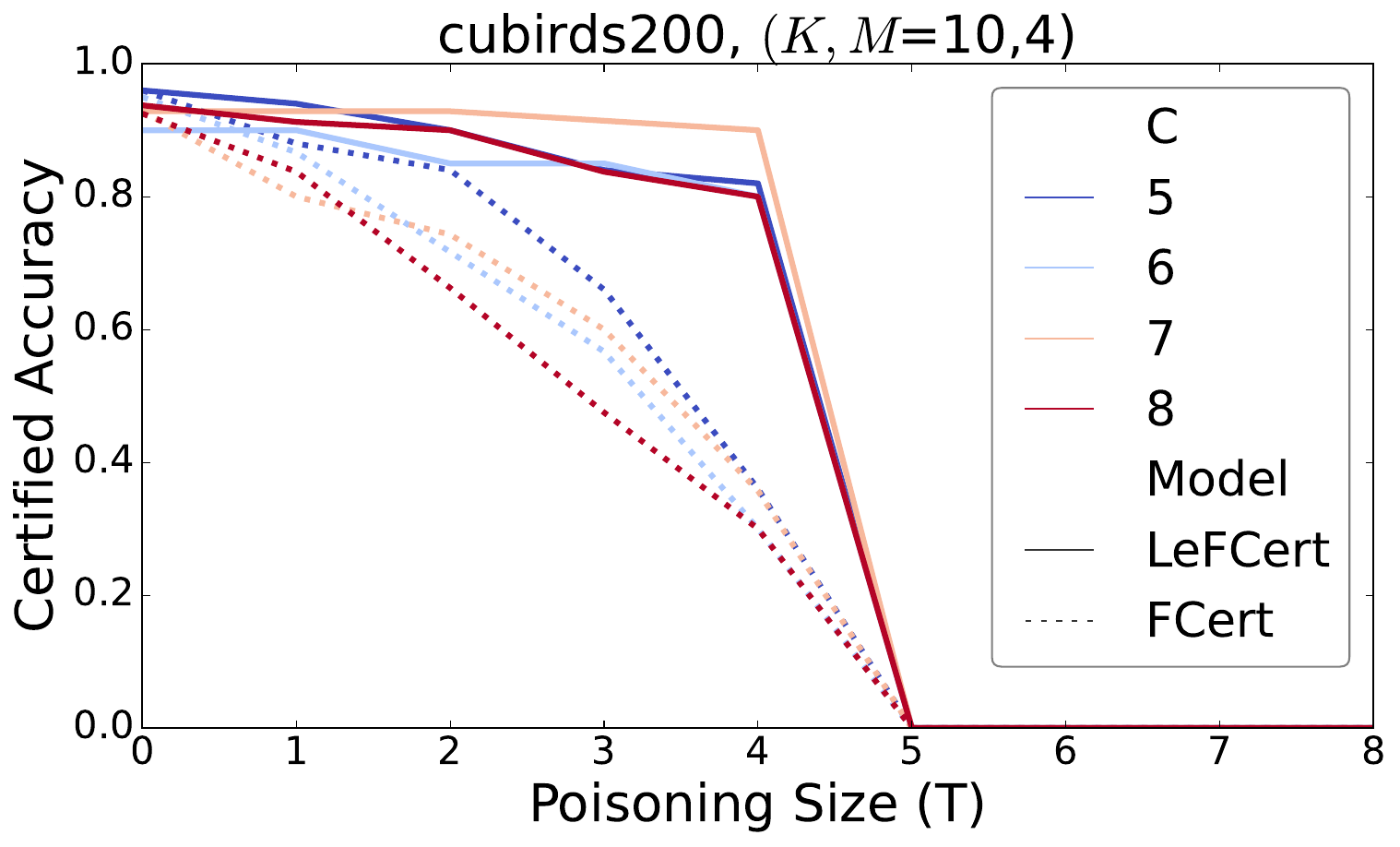}}
    \subfigure[Tiered-ImageNet]{\includegraphics[width=0.235\textwidth,height=2.95cm]{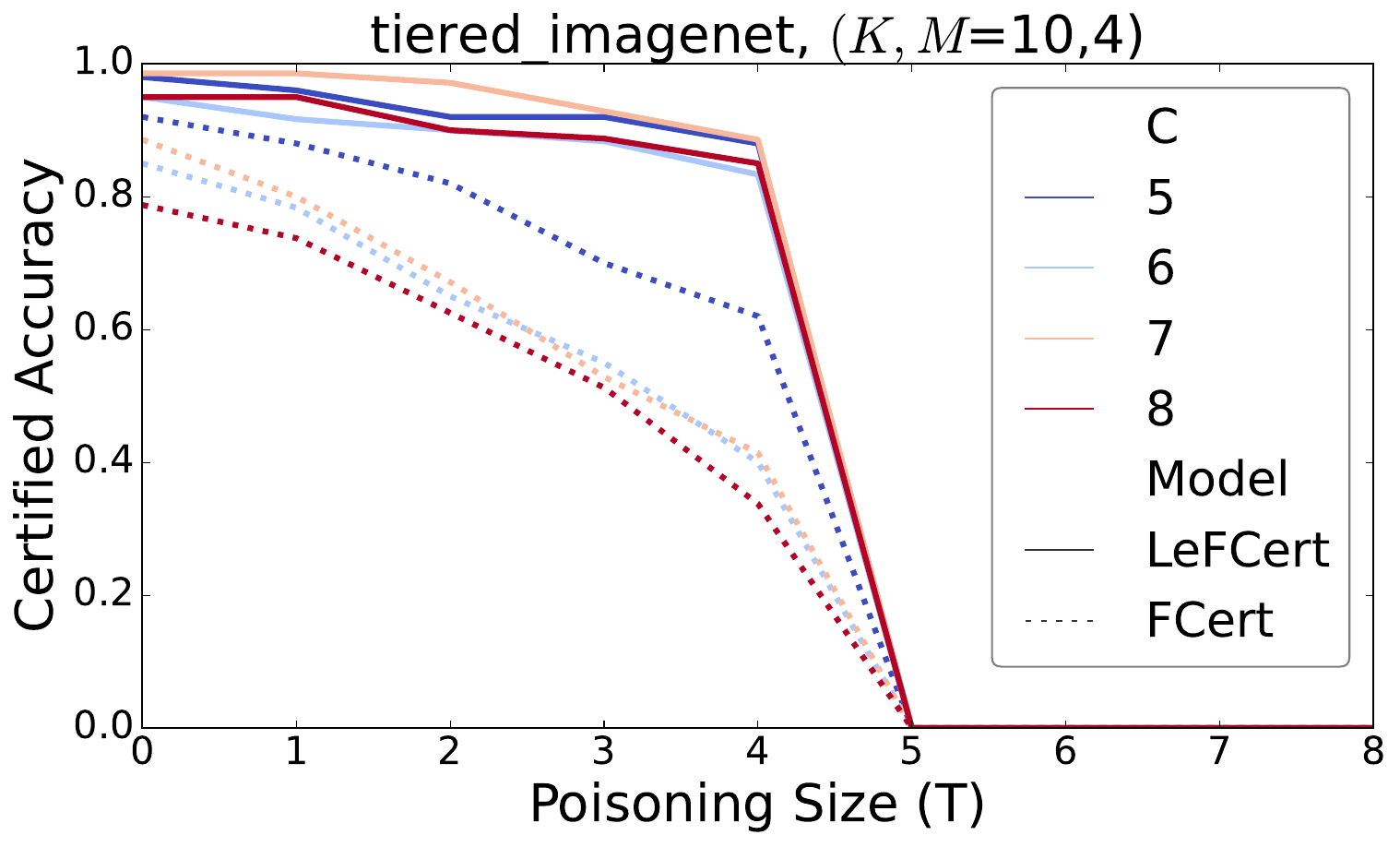}}
    \subfigure[CUB200-2011]{\includegraphics[width=0.235\textwidth,height=2.95cm]{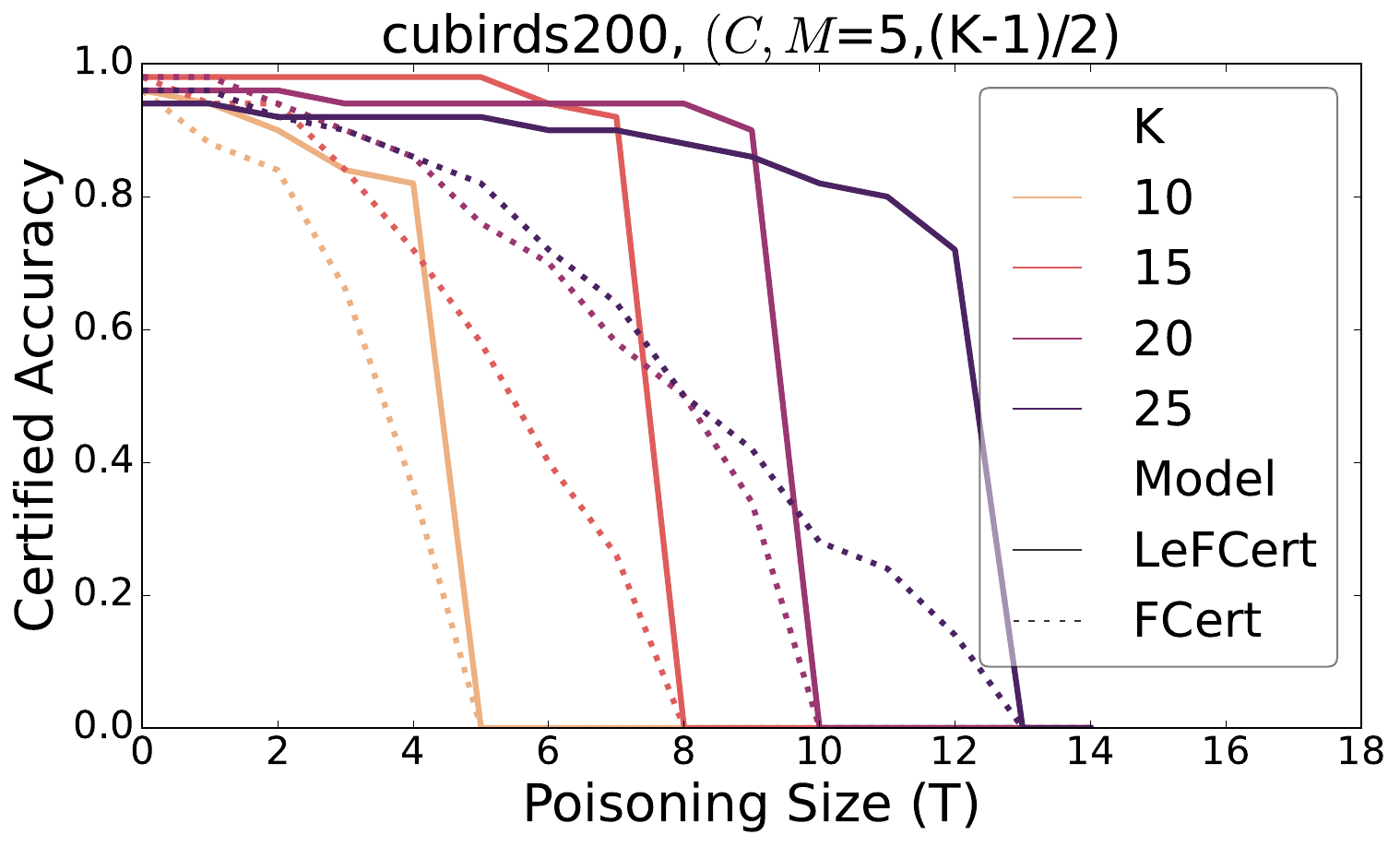}}
    \subfigure[Tiered-ImageNet]{\includegraphics[width=0.235\textwidth,height=2.95cm]{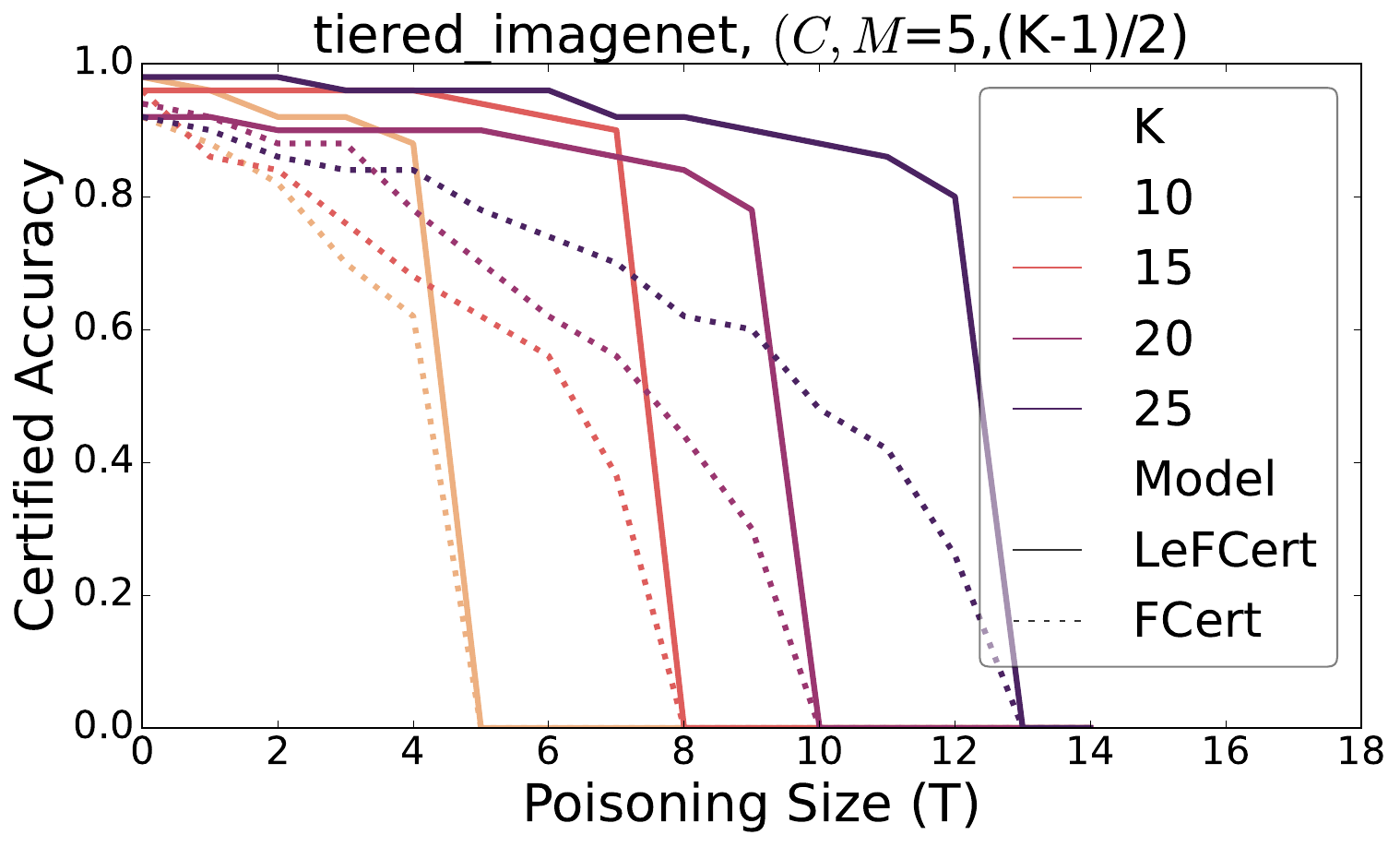}}
    \caption{Impact of few-shot settings: $C$ and $K$. \textbf{Across all settings, LeFCert consistently outperforms FCert by a significant margin, especially for larger $C$ and $K$.} }
\label{fig:images_CK}
\end{figure}

\begin{figure}[!ht]
    \centering
    \subfigure[CIFAR-FS]{\includegraphics[width=0.235\textwidth,height=2.95cm]{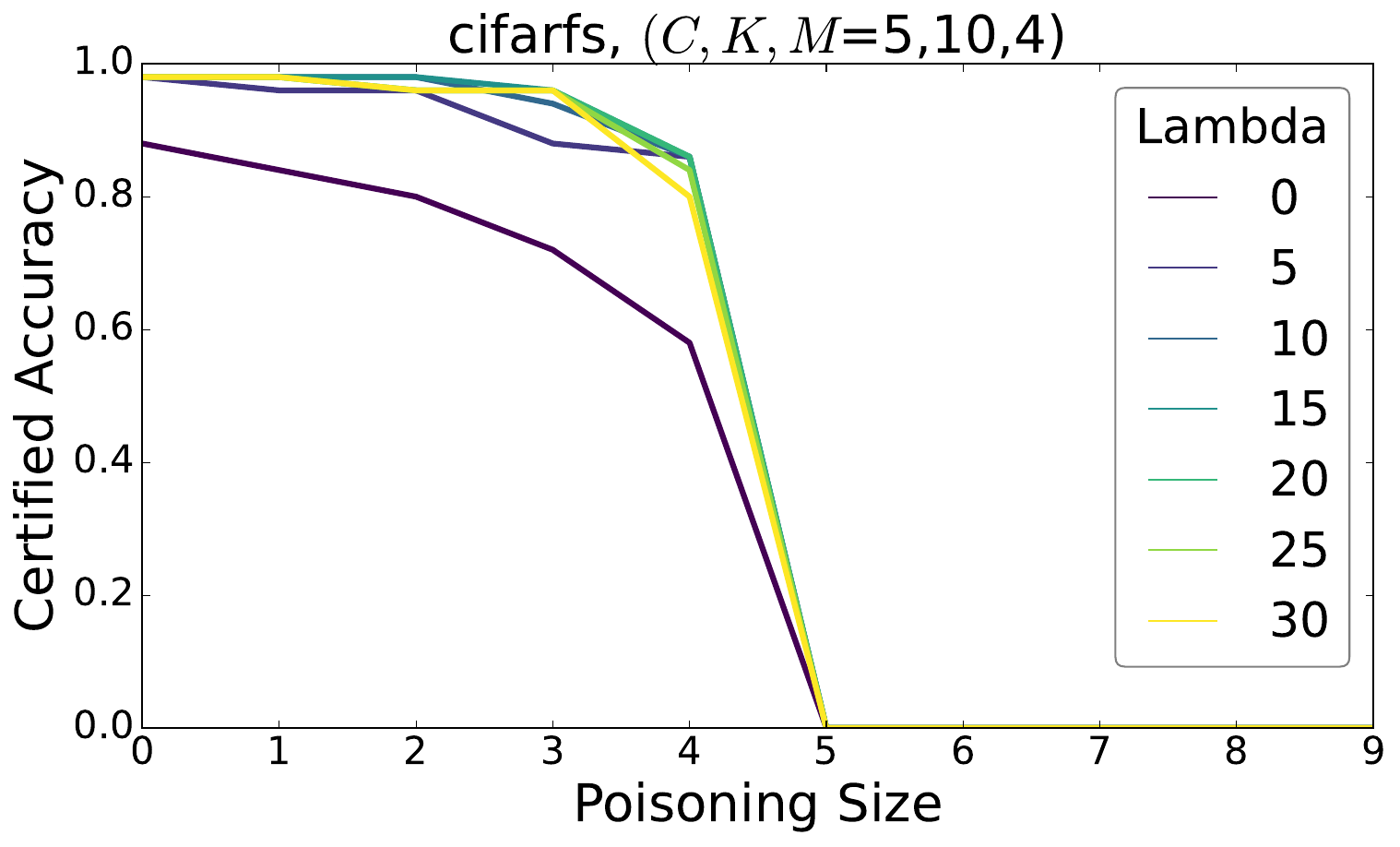}}
    \subfigure[Tiered-ImageNet]{\includegraphics[width=0.235\textwidth,height=2.95cm]{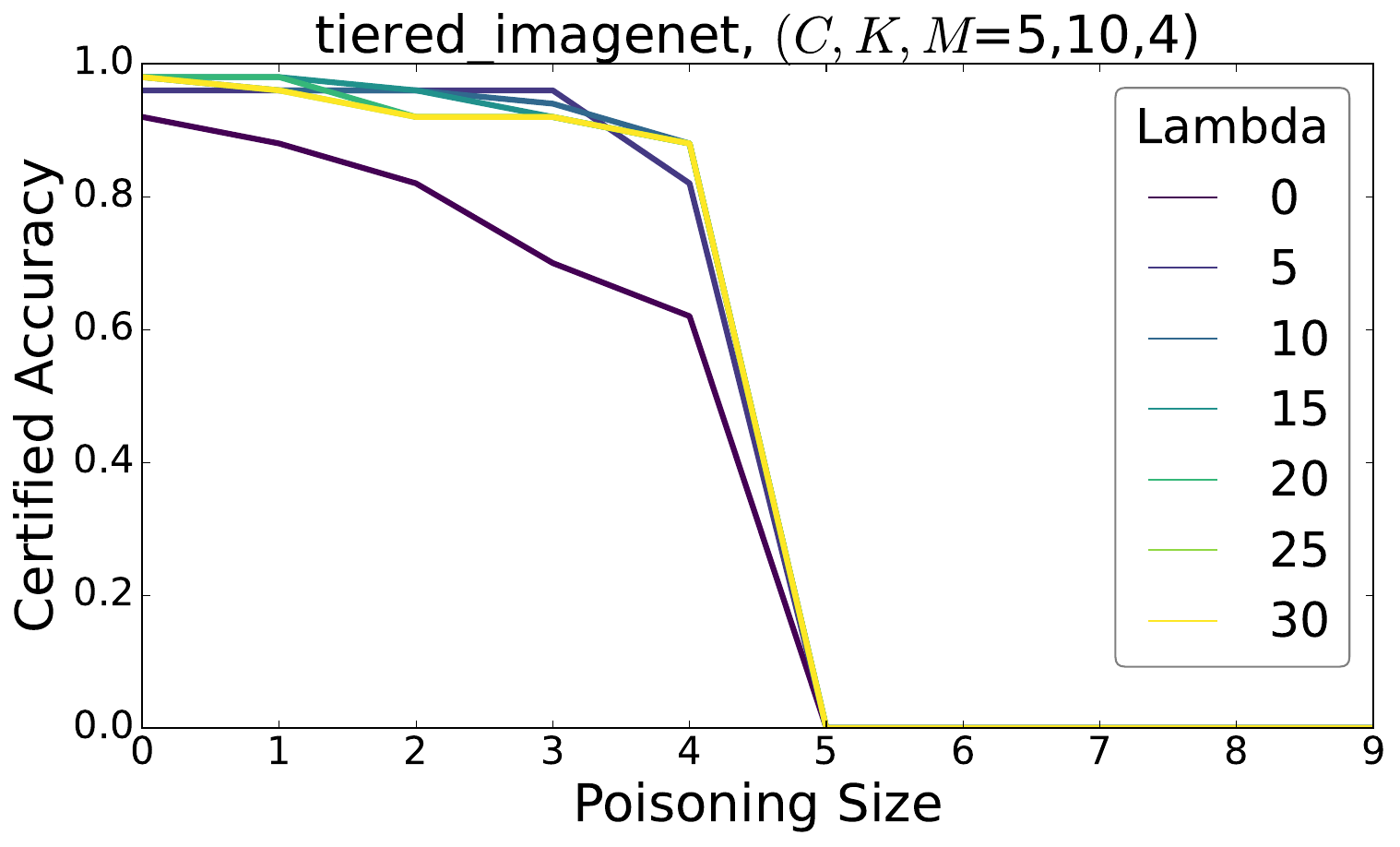}}
    \caption{Impact of hyperparameter $\lambda$. When $\lambda = 0$, LeFCert reduces to FCert, resulting in a significant drop in both clean and certified accuracy. \textbf{This study highlights the critical role of textual information in enhancing robustness}.}
\label{fig:images_Lambda}
\end{figure}

\subsubsection{Parameter Analysis}

To address \textbf{Q4}, we analyze the impact of various parameters on the performance of LeFCert, including the trimming parameter $M$, the blending parameter $\lambda$, the number of classes $C$, and the number of shots $K$. The results highlight the superior adaptability and robustness of LeFCert across diverse configurations. 

\textbf{Trimming parameter $M$:} 
The trimming parameter $M$ determines the number of outlier distances excluded during robust classification. As shown in Figure~\ref{fig:Tradeoff}, the number within the red circle represents the value of $M$ for LeFCert. We find that setting $M = \lfloor (K-1)/2\rfloor$ consistently provides optimal performance, achieving high clean and certified accuracy simultaneously.

\textbf{$C$-way $K$-shot:} In Figure~\ref{fig:images_CK}, we evaluate LeFCert under various few-shot configurations by varying the number of classes ($C$) and shots per class ($K$). Across all settings, LeFCert consistently outperforms FCert by a significant margin. Interestingly, as the number of classes increases, the gap in certified accuracy between LeFCert and FCert widens. Similarly, with more support samples per class, the certified accuracy gap grows larger. These trends demonstrate that LeFCert effectively scales to handle more complex classification tasks (higher $C$) and benefits from additional support samples (higher $K$). This superior scalability highlights the strength of LeFCert’s integration of textual and feature embeddings, which enhances its robustness even in challenging few-shot adaptation.

\textbf{Blending parameter $\lambda$:} It is a hyperparameter that blends the feature and textual information, and we evaluate various $\lambda$ in Figure~\ref{fig:images_Lambda}. We observe that LeFCert is not sensitive to the $\lambda$, allowing us to use the same value across all datasets in our experiments. This makes the hyperparameter easy to tune. However, when $\lambda = 0$, LeFCert reduces to FCert, resulting in a significant drop in both clean and certified accuracy. This ablation study highlights the critical role of textual information in enhancing robustness. The seamless integration of textual and feature embeddings in LeFCert provides a substantial advantage over FCert, emphasizing the importance of leveraging hybrid information for robust few-shot learning.

\section{Conclusion}
In this paper, we address the critical challenge of ensuring certified robustness for few-shot classifiers based on language-empowered foundation models (LeFMs). While LeFMs, such as CLIP and GraphCLIP, have demonstrated remarkable performance in multimodal learning and few-shot adaptation, their reliance on small, task-specific support sets leaves them vulnerable to poisoning attacks. Existing empirical defenses lack formal guarantees, and prior certification methods fail to leverage the hybrid embedding space of LeFMs, leaving a significant gap in robust few-shot learning.

To bridge this gap, we propose LeFCert, a novel certification framework that integrates both textual and feature embeddings to provide provable robustness guarantees. By employing an adaptive blending mechanism and a twofold trimmed mean approach, LeFCert mitigates the influence of adversarial samples while maintaining high clean accuracy. Additionally, to further enhance the certification performance, we extend LeFCert with two variants: LeFCert-L, which incorporates Lipschitz continuity to handle dual budget constraints, and LeFCert-C, which provides collective certification for multiple testing samples under shared adversarial budgets. These extensions address more realistic threat models and enhance the practical robustness of our framework.

Through extensive experiments on standard benchmarks across vision and graph domains, we demonstrate that LeFCert and its variants consistently achieve state-of-the-art performance, significantly improving both clean and certified accuracy compared to existing baselines. Notably, on CIFAR-FS, LeFCert achieves $98\%$ clean accuracy and $96\%$ certified accuracy when poisoning size $T=3$.
LeFCert-LD achieves robust certification in dual-constraint scenarios, while LeFCert-C provides stronger guarantees in collective adversarial settings. Importantly, LeFCert maintains computational efficiency, making it feasible for real-world applications.

Our contributions establish a new benchmark for robust few-shot learning with foundation models, paving the way for secure and reliable downstream applications of LeFMs in diverse domains. Future work may explore extending LeFCert to dynamic adaptation scenarios, broader multimodal datasets, and real-time certification frameworks to further enhance its applicability and impact.

\appendix
\cleardoublepage

\section*{Ethical Considerations}
This work addresses the critical challenge of ensuring certified robustness for few-shot classifiers based on language-empowered foundation models (LeFMs). Our proposed framework, \textbf{LeFCert}, enhances the reliability and security of LeFMs, particularly in safety-critical applications, by mitigating the impact of poisoning attacks. Below, we summarize the ethical considerations of this research.

\subsubsection*{Stakeholder Analysis}
Our research impacts the following stakeholders:
\begin{itemize}
    \item \textbf{Users of LeFMs:} LeFCert provides enhanced robustness against adversarial attacks, benefiting individuals and organizations relying on LeFMs for applications such as healthcare, autonomous systems, and content moderation.
    \item \textbf{Developers and researchers:} LeFCert equips practitioners with a framework for certifying robustness, advancing the state of security in AI.
    \item \textbf{Society at large:} By improving the trustworthiness of AI systems, our work promotes safer and more reliable applications of LeFMs in diverse domains.
\end{itemize}

\subsubsection*{Ethical Principles}
Our research adheres to the principles outlined in \textit{The Menlo Report}:
\begin{itemize}
    \item \textbf{Beneficence:} LeFCert mitigates risks from adversarial attacks, promoting societal welfare and improving the reliability of AI systems.
    \item \textbf{Respect for Persons:} We use publicly available datasets, avoiding privacy violations or harm to individuals.
    \item \textbf{Justice:} Our framework is broadly applicable across domains, ensuring equitable access to robust AI technologies.
    \item \textbf{Respect for Law and Public Interest:} The research complies with ethical guidelines and aims to advance public trust in AI.
\end{itemize}

\subsubsection*{Potential Harms and Mitigations}
\begin{itemize}
    \item \textbf{Misuse by adversaries:} While publishing robustness techniques can theoretically inform adversaries, the benefits of improving security for legitimate users far outweigh this risk. Our framework focuses on defensive strategies and \textbf{does not} expose vulnerabilities or attack methods that could be exploited maliciously.

    \item \textbf{Algorithmic misinterpretation:} To mitigate misconceptions, we explicitly clarify the limitations of our certification guarantees, ensuring realistic expectations.
\end{itemize}

\subsubsection*{Decision to Proceed and Publish}
The potential societal benefits of this work, including enhanced AI security and reliability, outweigh the minimal risks. Our research respects the rights and interests of all stakeholders, avoids the violation of human rights, and adheres to ethical principles. The use of publicly available datasets ensures that no individuals or groups are directly harmed. Publishing this research promotes transparency in security-focused AI research and provides a foundation for future advancements in certified robustness. It empowers researchers and practitioners to build more secure systems, benefiting both individual users and society.
By adhering to ethical principles and focusing on public interest, we ensure this research contributes positively to secure and trustworthy AI systems.

\subsubsection*{Conclusion}
LeFCert establishes a novel benchmark for certified robustness in few-shot classifiers, addressing a critical gap in secure AI research. By adhering to ethical standards and mitigating harms, we are confident that the publication of this work will contribute positively to the field of AI security and foster the development of more robust and trustworthy AI systems.

\section*{Open Science}
We conduct the experiments on the environment: Intel(R) Xeon(R) Platinum 8369B CPU @ 2.90GHz, a single NVIDIA GeForce RTX3090 24GB GPU, Ubuntu 20.04, Python 3.7, PyTorch 1.13.1. 
All the source code, datasets, and detailed environment settings for this paper are available at an anonymous GitHub Repo: https://anonymous.4open.science/r/LeFCert.

\bibliographystyle{plain}
\bibliography{references}

\begin{thebibliography}{10}

\bibitem{alain2017understanding}
Guillaume Alain and Yoshua Bengio.
\newblock Understanding intermediate layers using linear classifier probes,
  2017.

\bibitem{alhussien2023novel}
Nour Alhussien and Ahmed Aleroud.
\newblock A novel poisoning attack on few-shot based network intrusion
  detection.
\newblock In {\em NOMS 2023-2023 IEEE/IFIP Network Operations and Management
  Symposium}, pages 1--5. IEEE, 2023.

\bibitem{bertinetto2018metalearning}
Luca Bertinetto, Joao~F. Henriques, Philip Torr, and Andrea Vedaldi.
\newblock Meta-learning with differentiable closed-form solvers.
\newblock In {\em International Conference on Learning Representations}, 2019.

\bibitem{carlini2023certified}
Nicholas Carlini, Florian Tramer, Krishnamurthy~Dj Dvijotham, Leslie Rice,
  Mingjie Sun, and J~Zico Kolter.
\newblock (certified!!) adversarial robustness for free!
\newblock In {\em The Eleventh International Conference on Learning
  Representations}, 2023.

\bibitem{Chase2022LangChain}
Harrison Chase.
\newblock {LangChain}, 2022.
\newblock Available at: https://github.com/langchain-ai/langchain.

\bibitem{cohen2019certified}
Jeremy Cohen, Elan Rosenfeld, and Zico Kolter.
\newblock Certified adversarial robustness via randomized smoothing.
\newblock In {\em international conference on machine learning}, pages
  1310--1320. PMLR, 2019.

\bibitem{deleu2019torchmeta}
Tristan Deleu, Tobias W\"urfl, Mandana Samiei, Joseph~Paul Cohen, and Yoshua
  Bengio.
\newblock {Torchmeta: A Meta-Learning library for PyTorch}.
\newblock {\em arXiv preprint arXiv:1909.06576}, 2019.
\newblock Available at: https://github.com/tristandeleu/pytorch-meta.

\bibitem{dhariwal2021diffusion}
Prafulla Dhariwal and Alexander Nichol.
\newblock Diffusion models beat gans on image synthesis.
\newblock {\em Advances in neural information processing systems},
  34:8780--8794, 2021.

\bibitem{dong2024adversarially}
Junhao Dong, Piotr Koniusz, Junxi Chen, Xiaohua Xie, and Yew-Soon Ong.
\newblock Adversarially robust few-shot learning via parameter co-distillation
  of similarity and class concept learners.
\newblock In {\em Proceedings of the IEEE/CVF Conference on Computer Vision and
  Pattern Recognition}, pages 28535--28544, 2024.

\bibitem{dong2022improving}
Junhao Dong, Yuan Wang, Jian-Huang Lai, and Xiaohua Xie.
\newblock Improving adversarially robust few-shot image classification with
  generalizable representations.
\newblock In {\em Proceedings of the IEEE/CVF Conference on Computer Vision and
  Pattern Recognition}, pages 9025--9034, 2022.

\bibitem{farina2025rethinking}
Matteo Farina, Massimiliano Mancini, Giovanni Iacca, and Elisa Ricci.
\newblock Rethinking few-shot adaptation of vision-language models in two
  stages.
\newblock In {\em Proceedings of the Computer Vision and Pattern Recognition
  Conference}, pages 29989--29998, 2025.

\bibitem{fu2023styleadv}
Yuqian Fu, Yu~Xie, Yanwei Fu, and Yu-Gang Jiang.
\newblock Styleadv: Meta style adversarial training for cross-domain few-shot
  learning.
\newblock In {\em Proceedings of the IEEE/CVF conference on computer vision and
  pattern recognition}, pages 24575--24584, 2023.

\bibitem{gao2024clip}
Peng Gao, Shijie Geng, Renrui Zhang, Teli Ma, Rongyao Fang, Yongfeng Zhang,
  Hongsheng Li, and Yu~Qiao.
\newblock Clip-adapter: Better vision-language models with feature adapters.
\newblock {\em International Journal of Computer Vision}, 132(2):581--595,
  2024.

\bibitem{ghiasvand2025few}
Sajjad Ghiasvand, Haniyeh~Ehsani Oskouie, Mahnoosh Alizadeh, and Ramtin
  Pedarsani.
\newblock Few-shot adversarial low-rank fine-tuning of vision-language models.
\newblock {\em arXiv preprint arXiv:2505.15130}, 2025.

\bibitem{giles1998citeseer}
C~Lee Giles, Kurt~D Bollacker, and Steve Lawrence.
\newblock Citeseer: An automatic citation indexing system.
\newblock In {\em Proceedings of the third ACM conference on Digital
  libraries}, pages 89--98, 1998.

\bibitem{goldblum2020adversarially}
Micah Goldblum, Liam Fowl, and Tom Goldstein.
\newblock Adversarially robust few-shot learning: A meta-learning approach.
\newblock {\em Advances in Neural Information Processing Systems},
  33:17886--17895, 2020.

\bibitem{huang2024lp++}
Yunshi Huang, Fereshteh Shakeri, Jose Dolz, Malik Boudiaf, Houda Bahig, and
  Ismail Ben~Ayed.
\newblock Lp++: A surprisingly strong linear probe for few-shot clip.
\newblock In {\em Proceedings of the IEEE/CVF Conference on Computer Vision and
  Pattern Recognition}, pages 23773--23782, 2024.

\bibitem{jia2021intrinsic}
Jinyuan Jia, Xiaoyu Cao, and Neil~Zhenqiang Gong.
\newblock Intrinsic certified robustness of bagging against data poisoning
  attacks.
\newblock In {\em Proceedings of the AAAI conference on artificial
  intelligence}, volume~35, pages 7961--7969, 2021.

\bibitem{jia2022certified}
Jinyuan Jia, Yupei Liu, Xiaoyu Cao, and Neil~Zhenqiang Gong.
\newblock Certified robustness of nearest neighbors against data poisoning and
  backdoor attacks.
\newblock In {\em Proceedings of the AAAI conference on artificial
  intelligence}, volume~36, pages 9575--9583, 2022.

\bibitem{jia2023pore}
Jinyuan Jia, Yupei Liu, Yuepeng Hu, and Neil~Zhenqiang Gong.
\newblock $\{$PORE$\}$: Provably robust recommender systems against data
  poisoning attacks.
\newblock In {\em 32nd USENIX Security Symposium (USENIX Security 23)}, pages
  1703--1720, 2023.

\bibitem{lai2024collective}
Yuni Lai, Bailin Pan, Kaihuang Chen, Yancheng Yuan, and Kai Zhou.
\newblock Collective certified robustness against graph injection attacks.
\newblock In {\em Proceedings of the 41st International Conference on Machine
  Learning}, pages 25871--25891, 2024.

\bibitem{lai2024node}
Yuni Lai, Yulin Zhu, Bailin Pan, and Kai Zhou.
\newblock Node-aware bi-smoothing: Certified robustness against graph injection
  attacks.
\newblock In {\em 2024 IEEE Symposium on Security and Privacy (SP)}, pages
  2958--2976. IEEE, 2024.

\bibitem{lecuyer2019certified}
Mathias Lecuyer, Vaggelis Atlidakis, Roxana Geambasu, Daniel Hsu, and Suman
  Jana.
\newblock Certified robustness to adversarial examples with differential
  privacy.
\newblock In {\em 2019 IEEE symposium on security and privacy (SP)}, pages
  656--672. IEEE, 2019.

\bibitem{levine2020randomized}
Alexander Levine and Soheil Feizi.
\newblock (de) randomized smoothing for certifiable defense against patch
  attacks.
\newblock {\em Advances in Neural Information Processing Systems},
  33:6465--6475, 2020.

\bibitem{levine2021deep}
Alexander Levine and Soheil Feizi.
\newblock Deep partition aggregation: Provable defenses against general
  poisoning attacks.
\newblock In {\em International Conference on Learning Representations}, 2021.

\bibitem{li2025agnncert}
Jiate Li and Binghui Wang.
\newblock Agnncert: Defending graph neural networks against arbitrary
  perturbations with deterministic certification.
\newblock {\em arXiv preprint arXiv:2502.00765}, 2025.

\bibitem{li2023sok}
Linyi Li, Tao Xie, and Bo~Li.
\newblock Sok: Certified robustness for deep neural networks.
\newblock In {\em 2023 IEEE symposium on security and privacy (SP)}, pages
  1289--1310. IEEE, 2023.

\bibitem{li2023libfewshot}
Wenbin Li, Ziyi Wang, Xuesong Yang, Chuanqi Dong, Pinzhuo Tian, Tiexin Qin,
  Jing Huo, Yinghuan Shi, Lei Wang, Yang Gao, et~al.
\newblock Libfewshot: A comprehensive library for few-shot learning.
\newblock {\em IEEE Transactions on Pattern Analysis and Machine Intelligence},
  45(12):14938--14955, 2023.
\newblock Available at: https://github.com/RL-VIG/LibFewShot.

\bibitem{li2025survey}
Zongxia Li, Xiyang Wu, Hongyang Du, Fuxiao Liu, Huy Nghiem, and Guangyao Shi.
\newblock A survey of state of the art large vision language models: Benchmark
  evaluations and challenges.
\newblock In {\em Proceedings of the Computer Vision and Pattern Recognition
  Conference}, pages 1587--1606, 2025.

\bibitem{liu2024does}
Xinwei Liu, Xiaojun Jia, Jindong Gu, Yuan Xun, Siyuan Liang, and Xiaochun Cao.
\newblock Does few-shot learning suffer from backdoor attacks?
\newblock In {\em Proceedings of the AAAI Conference on Artificial
  Intelligence}, volume~38, pages 19893--19901, 2024.

\bibitem{nichol2021improved}
Alexander~Quinn Nichol and Prafulla Dhariwal.
\newblock Improved denoising diffusion probabilistic models.
\newblock In {\em International conference on machine learning}, pages
  8162--8171. PMLR, 2021.

\bibitem{oldewage2021attacking}
Elre~Talea Oldewage, John~F Bronskill, and Richard~E Turner.
\newblock Attacking few-shot classifiers with adversarial support poisoning.
\newblock In {\em ICML 2021 Workshop on Adversarial Machine Learning}, 2021.

\bibitem{oldewage2022adversarial}
Elre~Talea Oldewage, John~F Bronskill, and Richard~E Turner.
\newblock Adversarial attacks are a surprisingly strong baseline for poisoning
  few-shot meta-learners.
\newblock In {\em I Can't Believe It's Not Better Workshop: Understanding Deep
  Learning Through Empirical Falsification}, 2022.

\bibitem{pautov2022smoothed}
Mikhail Pautov, Olesya Kuznetsova, Nurislam Tursynbek, Aleksandr Petiushko, and
  Ivan Oseledets.
\newblock Smoothed embeddings for certified few-shot learning.
\newblock {\em Advances in Neural Information Processing Systems},
  35:24367--24379, 2022.

\bibitem{radford2021learning}
Alec Radford, Jong~Wook Kim, Chris Hallacy, Aditya Ramesh, Gabriel Goh,
  Sandhini Agarwal, Girish Sastry, Amanda Askell, Pamela Mishkin, Jack Clark,
  et~al.
\newblock Learning transferable visual models from natural language
  supervision.
\newblock In {\em International conference on machine learning}, pages
  8748--8763. PmLR, 2021.

\bibitem{ren2019incremental}
Mengye Ren, Renjie Liao, Ethan Fetaya, and Richard Zemel.
\newblock Incremental few-shot learning with attention attractor networks.
\newblock {\em Advances in neural information processing systems}, 32, 2019.

\bibitem{sen2008collective}
Prithviraj Sen, Galileo Namata, Mustafa Bilgic, Lise Getoor, Brian Galligher,
  and Tina Eliassi-Rad.
\newblock Collective classification in network data.
\newblock {\em AI magazine}, 29(3):93--93, 2008.

\bibitem{shafahi2018poison}
Ali Shafahi, W~Ronny Huang, Mahyar Najibi, Octavian Suciu, Christoph Studer,
  Tudor Dumitras, and Tom Goldstein.
\newblock Poison frogs! targeted clean-label poisoning attacks on neural
  networks.
\newblock {\em Advances in neural information processing systems}, 31, 2018.

\bibitem{silva2024closer}
Julio Silva-Rodriguez, Sina Hajimiri, Ismail Ben~Ayed, and Jose Dolz.
\newblock A closer look at the few-shot adaptation of large vision-language
  models.
\newblock In {\em Proceedings of the IEEE/CVF Conference on Computer Vision and
  Pattern Recognition}, pages 23681--23690, 2024.

\bibitem{snell2017prototypical}
Jake Snell, Kevin Swersky, and Richard Zemel.
\newblock Prototypical networks for few-shot learning.
\newblock {\em Advances in neural information processing systems}, 30, 2017.

\bibitem{song2023comprehensive}
Yisheng Song, Ting Wang, Puyu Cai, Subrota~K Mondal, and Jyoti~Prakash Sahoo.
\newblock A comprehensive survey of few-shot learning: Evolution, applications,
  challenges, and opportunities.
\newblock {\em ACM Computing Surveys}, 55(13s):1--40, 2023.

\bibitem{tian2022comprehensive}
Zhiyi Tian, Lei Cui, Jie Liang, and Shui Yu.
\newblock A comprehensive survey on poisoning attacks and countermeasures in
  machine learning.
\newblock {\em ACM Computing Surveys}, 55(8):1--35, 2022.

\bibitem{triantafillou2019meta}
Eleni Triantafillou, Tyler Zhu, Vincent Dumoulin, Pascal Lamblin, Utku Evci,
  Kelvin Xu, Ross Goroshin, Carles Gelada, Kevin Swersky, Pierre-Antoine
  Manzagol, et~al.
\newblock Meta-dataset: A dataset of datasets for learning to learn from few
  examples.
\newblock {\em arXiv preprint arXiv:1903.03096}, 2019.

\bibitem{wang2022lethal}
Wenxiao Wang, Alexander Levine, and Soheil Feizi.
\newblock Lethal dose conjecture on data poisoning.
\newblock {\em Advances in neural information processing systems},
  35:1776--1789, 2022.

\bibitem{wang2024fcert}
Yanting Wang, Wei Zou, and Jinyuan Jia.
\newblock Fcert: Certifiably robust few-shot classification in the era of
  foundation models.
\newblock In {\em 2024 IEEE Symposium on Security and Privacy (SP)}, pages
  2939--2957. IEEE, 2024.

\bibitem{xu2021yet}
Han Xu, Yaxin Li, Xiaorui Liu, Hui Liu, and Jiliang Tang.
\newblock Yet meta learning can adapt fast, it can also break easily.
\newblock In {\em Proceedings of the 2021 SIAM International Conference on Data
  Mining (SDM)}, pages 540--548. SIAM, 2021.

\bibitem{zhang2024vision}
Jingyi Zhang, Jiaxing Huang, Sheng Jin, and Shijian Lu.
\newblock Vision-language models for vision tasks: A survey.
\newblock {\em IEEE transactions on pattern analysis and machine intelligence},
  46(8):5625--5644, 2024.

\bibitem{zhang2022tip}
Renrui Zhang, Wei Zhang, Rongyao Fang, Peng Gao, Kunchang Li, Jifeng Dai,
  Yu~Qiao, and Hongsheng Li.
\newblock Tip-adapter: Training-free adaption of clip for few-shot
  classification.
\newblock In {\em European conference on computer vision}, pages 493--510.
  Springer, 2022.

\bibitem{zhou2022learning}
Kaiyang Zhou, Jingkang Yang, Chen~Change Loy, and Ziwei Liu.
\newblock Learning to prompt for vision-language models.
\newblock {\em International Journal of Computer Vision}, 130(9):2337--2348,
  2022.

\bibitem{zhou2024fewshot}
Yiwei Zhou, Xiaobo Xia, Zhiwei Lin, Bo~Han, and Tongliang Liu.
\newblock Few-shot adversarial prompt learning on vision-language models.
\newblock In {\em The Thirty-eighth Annual Conference on Neural Information
  Processing Systems}, 2024.

\bibitem{zhou2024few}
Yiwei Zhou, Xiaobo Xia, Zhiwei Lin, Bo~Han, and Tongliang Liu.
\newblock Few-shot adversarial prompt learning on vision-language models.
\newblock {\em Advances in Neural Information Processing Systems},
  37:3122--3156, 2024.

\bibitem{zhu2024graphclip}
Yun Zhu, Haizhou Shi, Xiaotang Wang, Yongchao Liu, Yaoke Wang, Boci Peng,
  Chuntao Hong, and Siliang Tang.
\newblock Graphclip: Enhancing transferability in graph foundation models for
  text-attributed graphs.
\newblock {\em arXiv preprint arXiv:2410.10329}, 2024.

\end{thebibliography}

\section{Proofs}
\label{sec:proofs}
\begin{lemma}
\label{thm:lemma}
    Let $\mathbf{s} = [s_1, s_2, \cdots, s_K]$ be the distance vector, and let $\mathbf{p} = [p_1, p_2, \cdots, p_K]$ be its sorted version in ascending order, where $p_1 \leq p_2 \leq \cdots \leq p_K$. The trimmed sum of $\mathbf{s}$ with parameter $M$ is defined as $R=\sum_{i=M+1}^{K-M} p_i$. Assuming that the attacker can modify arbitrarily $T$ elements among $\mathbf{s}$ to minimize the $R$, the optimal strategy is to replace the $T$ largest distances with the smallest possible value (e.g., $0$).
\end{lemma}
\begin{proof}

The attacker's goal is to minimize $R$ by modifying $T$ elements of $\mathbf{s}$. After the modifications, the distance vector $\mathbf{s}$ is updated, and the vector $\mathbf{p}$ is re-sorted. This re-sorting affects the trimmed sum $R$ because the trimmed range $\mathbf{p}_{\text{trimmed}} = [p_{M+1}, p_{M+2}, \cdots, p_{K-M}]$ depends on the sorted values. We aim to prove that replacing the $T$ largest distances in $\mathbf{s}$ with the smallest possible value is optimal. 

\textbf{1) Optimal target value is the smallest value.}
 Firstly, we prove that once the $T$ elements (to be modified) are decided, the optimal solution is to modify them into the smallest values. Let $p_{m1},p_{m2},\cdots,p_{mT}$ denote the selected $T$ elements, and $p'_{m1},p'_{m2},\cdots,p'_{mT}$ denote the value after modification. 
 The trimmed sum $\sum_{i=M+1}^{K-M} p_i$ depends on the elements in the trimmed range $\mathbf{p}_{\text{trimmed}} = [p_{M+1}, p_{M+2}, \cdots, p_{K-M}]$ after sorting. 
 Obviously, increasing any element of $ p'_{mi}$ neither decreases the sum nor shifts the trimmed range $\mathbf{p}_{\text{trimmed}}$ further to the left. That is, increasing any element of $ p'_{mi}$ can not reduce $R$. 
 As a result, modifying it to the smallest possible value (e.g., $0$) ensures the minimal $R$.

\textbf{2) Optimal element selection is the largest $T$ elements.}

Next, we prove that our selection that modifying the $T$ largest elements in $\mathbf{p}$: $\{p_{K-T},\cdots,p_{K}\}$ is optimal. If there is an optimal selection that differs from our selection, we denote it as $\{p_{o1},p_{o2},\cdots,p_{oT}\}$. Assume that there are $l$ different elements between our selection and the optimal selection. After modification and sorting, we know that there are exactly $l$ elements that are different because all the selected elements are modified into the smallest value. Furthermore, after sorting, the $l$ elements in the optimal strategy are equal to or larger than our strategy, because our selection is the $T$ largest elements, and any other selected elements should be smaller than or equal to our selected elements. We know that, given any sequence $[s_1, s_2, \cdots, s_K]$, increasing any element of $s_i$ neither decreases the sum nor shifts the trimmed range $\mathbf{p}_{\text{trimmed}}$ further to the left. That is, increasing any element of $s_i$ can not reduce $R$. Then, our strategy must obtain the minimal $R$.

For example, given a sequence $\{1,2,3,4,5\}$, and $T=2$, our strategy is to select $4,5$: $\{1,2,3,4,5\}\rightarrow\{0,0,1,2,3\}$. Assume that the optimal strategy has one element that differs from ours. If the different element is $4$, then the selected element must be smaller than $4$; it should be among $\{1,2,3\}$. Whichever it chooses, it is smaller than $4$. Because it did not choose $4$, then the $4$ is remain unchanged in the sequence: $\{1,2,3,4,5\}\rightarrow\{0,0,2,3,4\}$. There is always one element greater than the list obtained by our strategy, and it could not obtain a smaller trimmed sum $R$.
\end{proof}

\setcounter{thm}{0}
\begin{thm}
(Restate) Let $p_1^c\geq p_2^c\geq\cdots\geq p_K^c$ denote the sorted sequence of $d(\mathbf{f}_{\text{test}},\mathbf{f}_{ci})$, $i=1,\cdots,K$, and $q_1^c\geq q_2^c \geq \cdots \geq q_K^c$ denote the sorted sequence of $d(\mathbf{f}_{ci},\mathbf{t}_c)$, $i=1,\cdots,K$. Let $R^c$ be the classification score defined in Eq.~\eqref{eqn:R^c}. Suppose the attacker can arbitrarily modify $T^c$ features of the support samples among $\{\mathbf{f}_{ci}|i=1,\cdots, K\}$. If $T^c\leq M$, we have the upper bound and lower bound of the classification score $R^c$ for a given class $c$ and a given perturbation size $T^c$:
$$\overline{R}^c(T^c)=\sum_{i=M+1+T^c}^{K-M+T^c} p_i^c +\frac{\lambda}{K-2M} (\sum_{i=M+1+T^c}^{K-M+T^c} q_i^c) d(\mathbf{f}_{\text{test}},\mathbf{t}_c).$$
$$\underline{R}^c(T^c)=\sum_{i=M+1-T^c}^{K-M-T^c} p_i^c +\frac{\lambda}{K-2M} (\sum_{i=M+1-T^c}^{K-M-T^c} q_i^c) d(\mathbf{f}_{\text{test}},\mathbf{t}_c).$$
\end{thm}

\begin{proof}
Because the distance values $p_i^c$, $q_i^c$, and $d(\mathbf{f}_{\text{test}},\mathbf{t}_c)$ are all non-negative, we can find the upper bound by upper bounding the first term and the second term separately. The upper (lower) bound of $R^c$ is achieved when the attacker maximizes (minimizes) each term in $R^c$, which consists of two components: the first term $\sum_{i=M+1}^{K-M} p_i^c$, and the numerator of the second term: $\sum_{i=M+1}^{K-M} q_i^c$.
From the lemma~\ref{thm:lemma}, we know that the attacker can maximize the trimmed sum by replacing the $T^c$ smallest elements of $\mathbf{p}^c$ with the largest possible values. After this modification, the new trimmed range is shifted to the right by $T^c$ indices. The attacker minimizes the trimmed sum by replacing the $T^c$ largest elements of $\mathbf{p}^c$ and $\mathbf{q}^c$ with the smallest possible values. After this modification, the new trimmed range is shifted to the left by $T^c$ indices.

\end{proof}

\begin{thm}
(Restate). Let $\mathcal{D}$ denote the clean support set for $C$-way $K$-shot few classification. The few-shot classifier $g$ is defined in \eqref{eqn:robust_classifier}. Let $\hat{y}=g(x_{\text{test}};\mathcal{D})$ represent the prediction for testing input $\text{test}$ with $\mathcal{D}$ as the support set.
We have a provably robust classification result:
\begin{align}
g(x_{\text{test}};\mathcal{D})=g(x_{\text{test}};\mathcal{D}_p),\, \forall \mathcal{D}_p\in \mathcal{B}(\mathcal{D},T),\nonumber
\end{align}
if:
\begin{align}
\overline{R}^{\hat{y}}(T^{\hat{y}})<\min_{c\neq \hat{y}} \underline{R}^c(T-T^{\hat{y}}), \forall T^{\hat{y}}: 0\leq T^{\hat{y}}\leq T.\nonumber
\end{align}

\end{thm}

\begin{proof}
The classifier $g$ predicts class $\hat{y}$ if and only if its classification score $R_{\hat{y}}$ is smaller than the scores $R^c$ for all other classes $c \neq \hat{y}$. Under poisoning, the robustness condition requires that the upper bound of $R_{\hat{y}}$ remains strictly smaller than the lower bound of $R^c$ for all $c \neq \hat{y}$. In~\eqref{eqn:certify_condition}, $\underline{R}^{\hat{y}}(T^{\hat{y}})$ is the lower bound of the score for class $\hat{y}$ when $T^{\hat{y}}$ samples in its support set are poisoned, and $\overline{R}^c(T - T^{\hat{y}})$ is the upper bound of the score for class $c$ when $T - T^{\hat{y}}$ samples in its support set are poisoned. 
The inequality ensures that even under the worst-case allocation of the attack budget $T$, the score for the true class $\hat{y}$ remains strictly smaller than the scores for all other classes $c \neq \hat{y}$. Thus, the classifier's prediction is robust to any poisoning attack within the predefined budget $T$.
\end{proof}

\paragraph{Probabilistic bounds for LeFCert-L and LeFCert-LD:}

In LeFCert-L and LeFCert-LD, we apply Gaussian noise to each input, and conduct Monte-Carlo to estimate the mean of the embedding $f_s(x) = \mathbb{E}_{\epsilon \sim \mathcal{N}(0, \sigma^2 I)} f(x + \epsilon)$. As a result, there is randomness in the mean embedding. We employ Hoeffding inequality to bound the mean embedding. Hoeffding inequality is known as:
$$P\left(\left|\frac{1}{n} \sum_{i=1}^n z_i - \mu\right| \geq t \right) \leq 2 \exp\left(-\frac{2nt^2}{(b-a)^2}\right),$$
where \(\mu = \mathbb{E}[z_i]\) is the true mean, \(n\) is the number of samples, \([a, b]\) is the interval for the random variable \(z_i\), \(t\) is the maximum deviation of the sample mean from the true mean.

To compute a confidence interval for the true mean \(\mu\), we set the right-hand side of Hoeffding's inequality to a predefined confidence level \(1-\alpha\). That is:
$$
2 \exp\left(-\frac{2nt^2}{(b-a)^2}\right) = \alpha.
$$
Taking the natural logarithm and solving for \(t\), we have:
\begin{equation}
\label{eqn:deviation}
    t = \sqrt{\frac{(b-a)^2 \ln(2/\alpha)}{2n}}=(b-a)\sqrt{\frac{ \ln(2/\alpha)}{2n}}.
\end{equation}

Thus, the confidence interval for \(\mu\) is given by:
$
\mu \in \left[\bar{z} - t, \bar{z} + t \right]
$, where \(\bar{z} = \frac{1}{n} \sum_{i=1}^n X_i\) is the sample mean, and $t$ is the deviation.

In our model, we normalize the output embedding such that $\|\mathcal{F}_{enc}(x)\|_2 = 1$. Then we have $||\mathbf{f}_{\text{test}}||_2=1$ and $||\mathbf{f}_{ci}||_2=1$. 
Let $p_i^c$ denote the $l_2$-norm distance $d(\mathbf{f}_{\text{test}},\mathbf{f}_{ci})=||\mathbf{f}_{\text{test}}-\mathbf{f}_{ci}||_2$, and $q_i^c$ denote the distance $d(\mathbf{f}_{ci},\mathbf{t}_c)=||\mathbf{f}_{ci}-\mathbf{t}_c||_2$, $i=1,\cdots,K$.

By the triangle inequality for the $l_2$ norm, we have $d(\mathbf{f}_{\text{test}},\mathbf{f}_{ci})=||\mathbf{f}_{\text{test}}-\mathbf{f}_{ci}||_2\leq ||\mathbf{f}_{\text{test}}||_2+||\mathbf{f}_{ci}||_2=2$. Then $b-a=2$. 
Then we have:
\begin{align}    d(\bar{\mathbf{f}}_{\text{test}},\bar{\mathbf{f}}_{ci})-t &\leq d(\mathbf{f}_{\text{test}},\mathbf{f}_{ci})\leq d(\bar{\mathbf{f}}_{\text{test}},\bar{\mathbf{f}}_{ci})+t.
\end{align}

We can also normalize the text embedding as $\|\mathbf{t}_c\|_2 = 1$, Then we have:
\begin{equation}
d(\bar{\mathbf{f}}_{ci},\mathbf{t}_c)-t \leq d(\mathbf{f}_{ci},\mathbf{t}_c)\leq d(\bar{\mathbf{f}}_{ci},\mathbf{t}_c)+t.
\end{equation}

\section{Other Experimental Results}
\label{sec:more_results}
Due to the space limit, we put more experimental results in this section. The results are discussed in the main paper. 
\begin{figure}[!ht]
    \centering
    \subfigure[CIFAR-FS]{\includegraphics[width=0.235\textwidth,height=2.95cm]{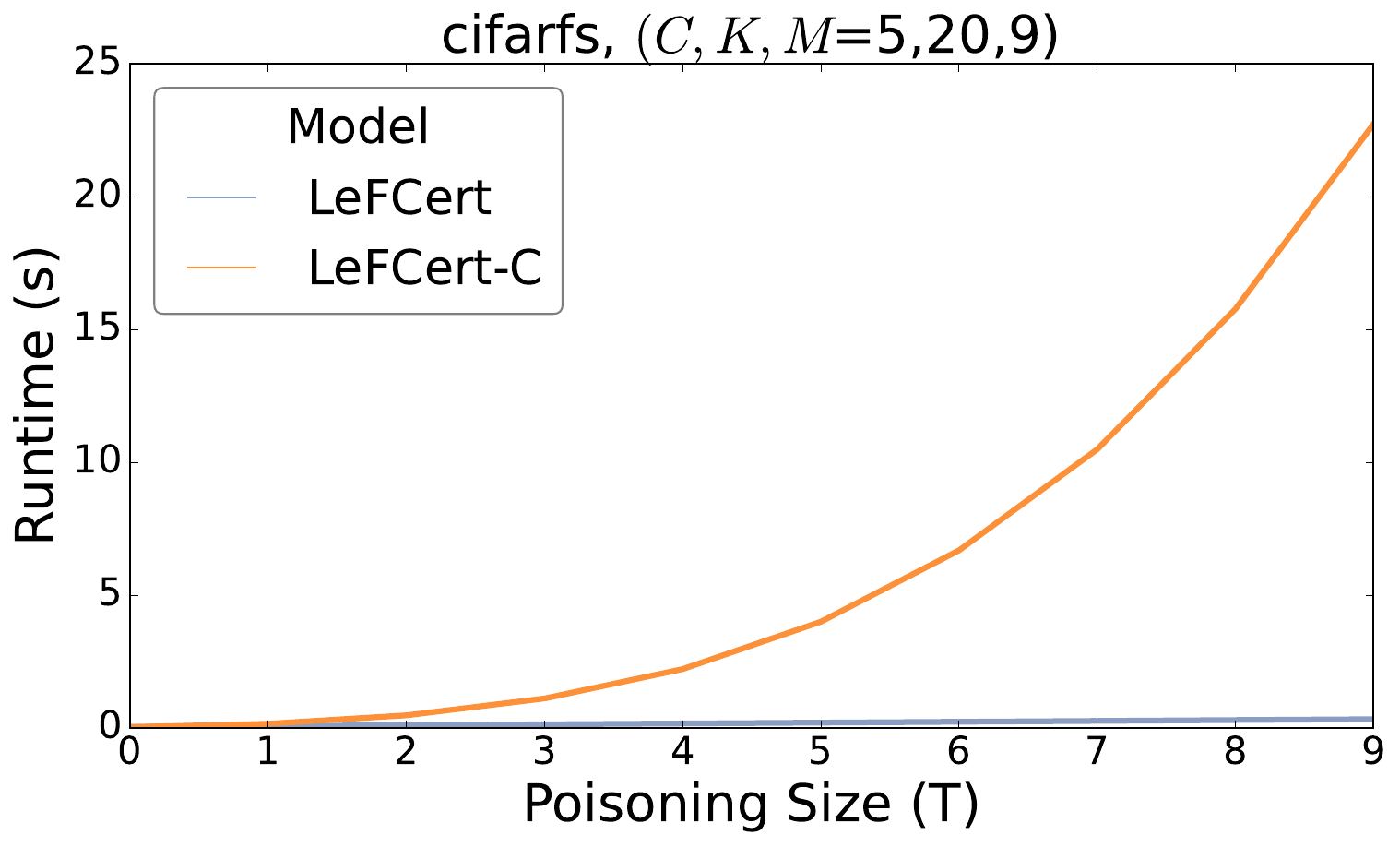}}
    \subfigure[Tiered-ImageNet]{\includegraphics[width=0.235\textwidth,height=2.95cm]{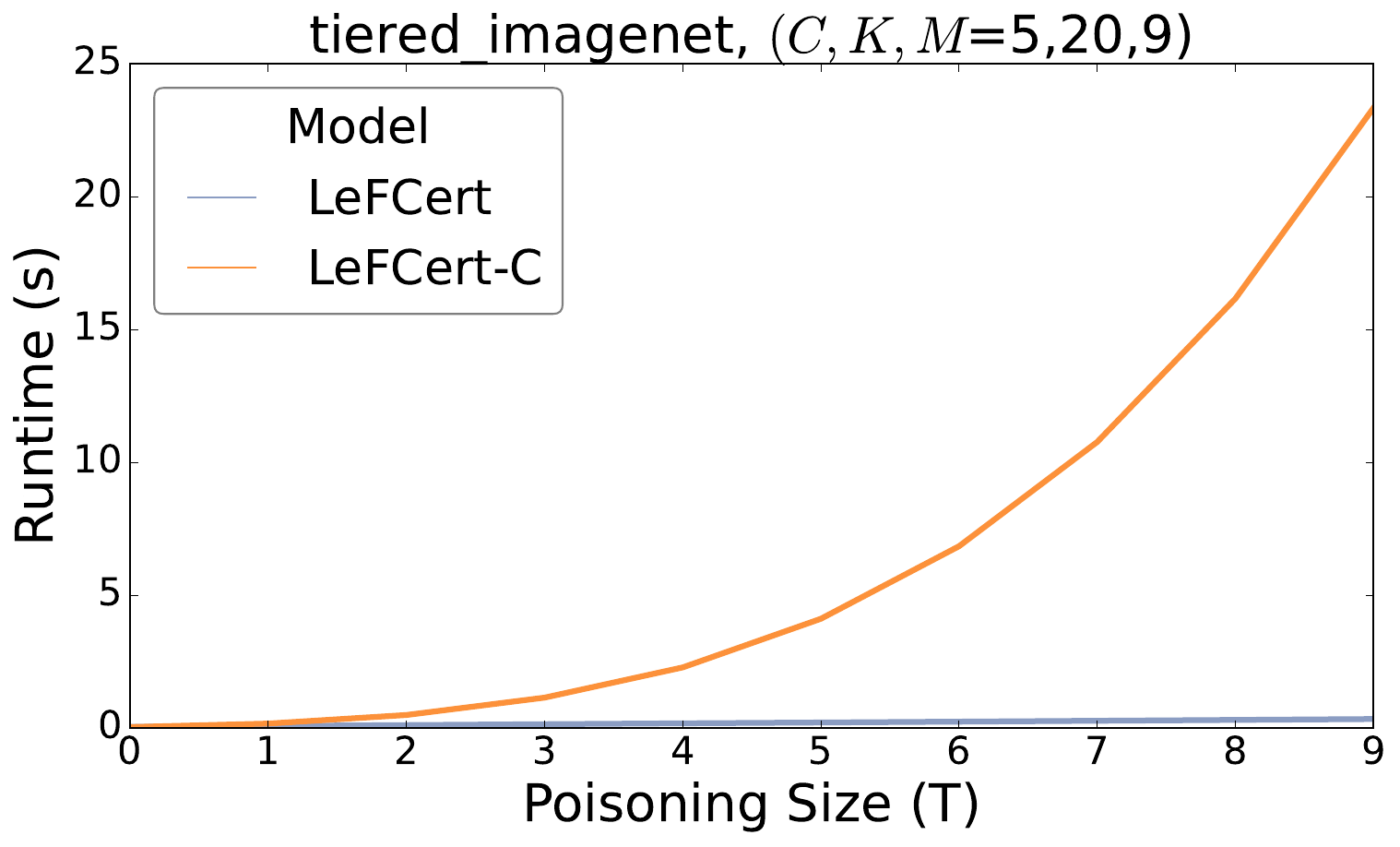}}
    \caption{Runtime of budget allocation algorithm in LeFCert-C. Although the runtime increases significantly with larger budgets, it is important to note that in few-shot scenarios, the number of shots $K$ is typically small, resulting in a limited budget range ($T \leq K-M-1$). \textbf{LeFCert-C is computationally practical for few-shot learning (less than 25s, K=20)}.}
\label{fig:time_Collect}
\end{figure}

\begin{figure}[!ht]
    \centering
    \includegraphics[width=1.0\linewidth]{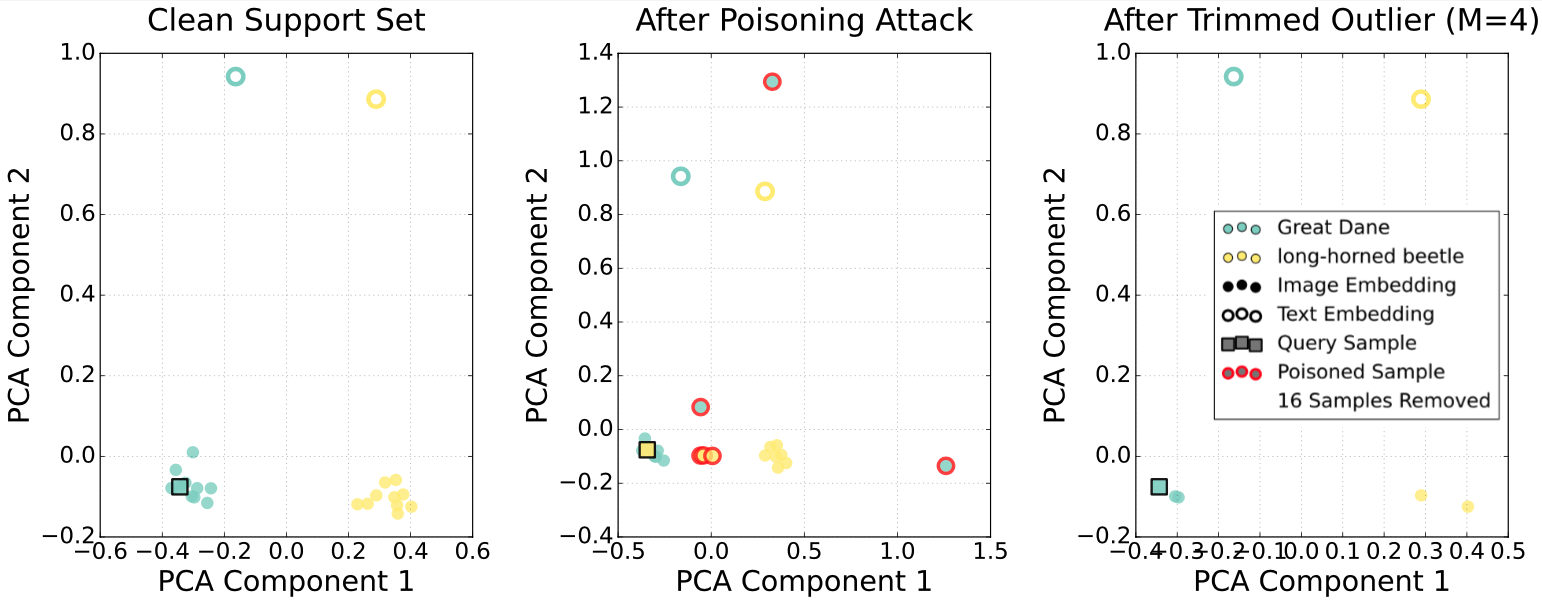}
    \caption{Feature and text embedding Principal Component Analysis (PCA) visualization: before attack, under attack, with LeFCert defense. We randomly select two classes from the Tiered-ImageNet datasets. We employ a simple poisoning attack on the support set with poisoning size $T=6$. Specifically, we select 3 of the support samples of the query class (Great Dane), modify their features to be far away from the query feature, and select 3 of the support samples of the target class (Long-horned beetle), modify their features to be close to the query feature. \textbf{After the poisoning attack, the query prediction is modified to the target class successfully}. \textbf{Nevertheless, with our LeFCert that trims the outlier, the query sample prediction is robust to the attack}. }
    \label{fig:pca}
\end{figure}
\end{document}